\documentclass[letterpaper]{article}

\usepackage[margin=1in]{geometry}
\usepackage{hyperref}
\usepackage{biblatex}
\usepackage{amsthm}
\usepackage{lipsum}
\usepackage{amsfonts}
\usepackage{graphicx}
\usepackage{epstopdf}
\usepackage{graphicx,epstopdf}
\usepackage[caption=false]{subfig}
\usepackage{algorithmic}

\usepackage{pgf}
\ifpdf
  \DeclareGraphicsExtensions{.eps,.pdf,.png,.jpg}
\else
  \DeclareGraphicsExtensions{.eps}
\fi

\newcommand{\scrbar}[1]{\overline{\mathcal{#1}}}

\ExplSyntaxOn

\cs_new_protected:Nn \bamboo_define:nnnnnN
{
	\cs_new_protected:cpx { #3 #1 #4 } { \exp_not:N #5{#6{#2}} }
}

\int_step_inline:nnn { `A } { `Z }
{
	\bamboo_define:nnnnnN
	{ \char_generate:nn { #1 } { 11 } }
	{ \char_generate:nn { #1 } { 11 } }
	{ bl }
	{ }
	{ \mymathbold }
	\use:n
	
	\bamboo_define:nnnnnN
	{ \char_generate:nn { #1 } { 11 } }
	{ \char_generate:nn { #1 } { 11 } }
	{ scr }
	{ }
	{ \mathcal }
	\use:n
	
	\bamboo_define:nnnnnN
	{ \char_generate:nn { #1 } { 11 } }
	{ \char_generate:nn { #1 } { 11 } }
	{ }
	{ hat }
	{ \widehat }
	\use:n
	
	\bamboo_define:nnnnnN
	{ \char_generate:nn { #1 } { 11 } }
	{ \char_generate:nn { #1 } { 11 } }
	{ }
	{ br }
	{ \overline }
	\use:n
	
	\bamboo_define:nnnnnN
	{ \char_generate:nn { #1 } { 11 } }
	{ \char_generate:nn { #1 } { 11 } }
	{ }
	{ tl }
	{ \widetilde }
	\use:n
	
	\bamboo_define:nnnnnN
	{ \char_generate:nn { #1 } { 11 } }
	{ \char_generate:nn { #1 } { 11 } }
	{ scr }
	{ br }
	{ \scrbar }
	\use:n
}
\int_step_inline:nnn { `a } { `z }
{
	\bamboo_define:nnnnnN
	{ \char_generate:nn { #1 } { 11 } }
	{ \char_generate:nn { #1 } { 11 } }
	{ bl }
	{ }
	{ \mymathbold }
	\use:n
	
	\bamboo_define:nnnnnN
	{ \char_generate:nn { #1 } { 11 } }
	{ \char_generate:nn { #1 } { 11 } }
	{ }
	{ hat }
	{ \hat }
	\use:n
	
	\bamboo_define:nnnnnN
	{ \char_generate:nn { #1 } { 11 } }
	{ \char_generate:nn { #1 } { 11 } }
	{ }
	{ br }
	{ \bar }
	\use:n
	
	\bamboo_define:nnnnnN
	{ \char_generate:nn { #1 } { 11 } }
	{ \char_generate:nn { #1 } { 11 } }
	{ }
	{ tl }
	{ \tilde }
	\use:n                            
}
\clist_map_inline:nn
{
	Gamma,Delta,Theta,Lambda,Xi,Pi,Sigma,Phi,Psi,Omega
}
{
	\bamboo_define:nnnnnN
	{ #1 }
	{ #1 }
	{ bl }
	{ }
	{ \mymathbold }
	\use:c
	
	\bamboo_define:nnnnnN
	{ #1 }
	{ #1 }
	{ op }
	{ }
	{ \op }
	\use:c
	
	\bamboo_define:nnnnnN
	{ #1 }
	{ #1 }
	{ scr }
	{ }
	{ \mathcal }
	\use:c
	
	\bamboo_define:nnnnnN
	{ #1 }
	{ #1 }
	{ }
	{ hat }
	{ \widehat }
	\use:c
	
	\bamboo_define:nnnnnN
	{ #1 }
	{ #1 }
	{ }
	{ br }
	{ \overline }
	\use:c
	
	\bamboo_define:nnnnnN
	{ #1 }
	{ #1 }
	{ }
	{ tl }
	{ \widetilde }
	\use:c
}
\clist_map_inline:nn
{
	alpha,beta,gamma,delta,epsilon,zeta,eta,theta,iota,kappa,
	lambda,mu,nu,xi,pi,rho,sigma,tau,phi,chi,psi,omega
}
{
	\bamboo_define:nnnnnN
	{ #1 }
	{ #1 }
	{ bl }
	{ }
	{ \mymathbold }
	\use:c
	
	\bamboo_define:nnnnnN
	{ #1 }
	{ #1 }
	{ }
	{ hat }
	{ \hat }
	\use:c
	
	\bamboo_define:nnnnnN
	{ #1 }
	{ #1 }
	{ }
	{ br }
	{ \bar }
	\use:c
	
	\bamboo_define:nnnnnN
	{ #1 }
	{ #1 }
	{ }
	{ tl }
	{ \tilde }
	\use:c
}

\ExplSyntaxOff
\usepackage{mathtools}
\usepackage{amssymb}

\makeatletter
\@ifpackageloaded{unicode-math}{%
	\newcommand{\mymathbold}{\symbf}%
}{%
	\usepackage{bm}%
	\newcommand{\mymathbold}{\bm}%
}
\makeatother

\DeclareMathOperator{\E}{\mathbf{E}}
\renewcommand{\P}{\operatorname{\mathbf{P}}}

\newcommand{\tr}{\operatorname{tr}}
\newcommand{\argmin}{\operatornamewithlimits{arg~min}}
\newcommand{\argmax}{\operatornamewithlimits{arg~max}}
\newcommand{\diag}{\operatorname{diag}}

\DeclarePairedDelimiter{\norm}{\lVert}{\rVert}

\DeclarePairedDelimiter{\abs}{\lvert}{\rvert}

\DeclarePairedDelimiter{\parens}{(}{)}

\DeclarePairedDelimiterX{\ip}[2]{\langle}{\rangle}{#1,#2}

\DeclarePairedDelimiterXPP{\normsub}[2]{}{\lVert}{\rVert}{_{#2}}{#1}
\DeclarePairedDelimiterXPP{\ipsub}[3]{}{\langle}{\rangle}{_{#3}}{#1,#2}

\DeclarePairedDelimiterXPP{\ipHS}[2]{}{\langle}{\rangle}{_{\mathrm{HS}}}{#1, #2}
\DeclarePairedDelimiterXPP{\normHS}[1]{}{\lVert}{\rVert}{_{\mathrm{HS}}}{#1}

\DeclarePairedDelimiterXPP{\ipF}[2]{}{\langle}{\rangle}{_{\mathrm{F}}}{#1, #2}
\DeclarePairedDelimiterXPP{\normF}[1]{}{\lVert}{\rVert}{_{\mathrm{F}}}{#1}

\DeclarePairedDelimiterXPP{\dkl}[2]{\operatorname{D_{KL}}}{(}{)}{}{#1 \: \delimsize\Vert \: #2}

\DeclarePairedDelimiterXPP{\restr}[2]{}{{}}{\vert}{_{#2}}{#1}

\newcommand{\R}{\mathbf{R}}

\newcommand{\spn}{\operatorname{span}}

\newcommand{\given}{\:\vert\:}
\newcommand{\sign}{\operatorname{sign}}
\newcommand{\simiid}{\overset{\mathclap{\text{i.i.d.}}}{\sim}}

\newcommand\inputpgf[2]{{
\let\pgfimageWithoutPath\pgfimage
\renewcommand{\pgfimage}[2][]{\pgfimageWithoutPath[##1]{#1/##2}}
\let\includegraphicsWithoutPath\includegraphics
\renewcommand{\includegraphics}[2][]{\includegraphicsWithoutPath[##1]{#1/##2}}
\input{#1/#2}
}}
\usepackage{enumitem}
\setlist[enumerate]{leftmargin=.5in}
\setlist[itemize]{leftmargin=.5in}

\newtheorem{lemma}{Lemma}
\newtheorem{theorem}{Theorem}
\newtheorem{corollary}{Corollary}
\newtheorem{proposition}{Proposition}

\newtheorem{assumption}{Assumption}

\usepackage{cleveref}
\Crefname{assumption}{Assumption}{Assumptions}

\title{New Equivalences Between Interpolation and SVMs: Kernels and Structured Features}

\author{Chiraag Kaushik\thanks{School of Electrical and Computer Engineering, Georgia Institute of Technology, Atlanta, GA (ckaushik7@gatech.edu, mdav@gatech.edu, vmuthukumar8@gatech.edu).}
\and Andrew D.\ McRae\thanks{Institute of Mathematics, EPFL, Lausanne, Switzerland 
  (andrew.mcrae@epfl.ch)},
\and Mark A.\ Davenport\footnotemark[2]
\and Vidya Muthukumar\footnotemark[2] \thanks{School of Industrial \& Systems Engineering, Georgia Institute of Technology, Atlanta, GA.}}

\usepackage{amsopn}

\ifpdf
\hypersetup{
  pdftitle={New Equivalences Between Interpolation and SVMs: Kernels and Structured Features},
  pdfauthor={C. Kaushik, A. McRae, M. Davenport, and V. Muthukumar}
}
\fi

\begin{document}

\maketitle

\begin{abstract}
The support vector machine (SVM) is a supervised learning algorithm that finds a maximum-margin linear classifier, often after mapping the data to a high-dimensional feature space via the kernel trick. Recent work has demonstrated that in certain sufficiently overparameterized settings, the SVM decision function coincides exactly with the minimum-norm label interpolant. This phenomenon of support vector proliferation (SVP) is especially interesting because it allows us to understand SVM performance by leveraging recent analyses of harmless interpolation in linear and kernel models. However, previous work on SVP has made restrictive assumptions on the data/feature distribution and spectrum. In this paper, we present a new and flexible analysis framework for proving SVP in an arbitrary reproducing kernel Hilbert space with a flexible class of generative models for the labels. We present conditions for SVP for features in the families of general bounded orthonormal systems (e.g. Fourier features) and independent sub-Gaussian features. In both cases, we show that SVP occurs in many interesting settings not covered by prior work, and we leverage these results to prove novel generalization results for kernel SVM classification.
\end{abstract}



\section{Introduction}\label{sec:introduction}
Recent empirical and theoretical efforts in supervised machine learning have discovered a wide range of surprising phenomena that arise in the modern \emph{overparameterized regime} (i.e., where the number of free parameters in the model is much larger than the number of training examples \cite{Dar2021, Belkin2021}). For example, after it was observed that deep neural networks can perfectly fit noisy training data and still generalise well to new data (see, e.g., \cite{neyshabur2014search, Zhang2017}), several theoretical efforts have demonstrated that this “harmless interpolation” phenomenon can in fact occur even in the simpler settings of linear and kernel regression~\cite{Belkin2018, Belkin2019a, Bartlett2020}. A separate, but equally surprising observation in this overparameterized regime is that training procedures that optimize different loss functions can still yield similar test performance. For example, the empirical studies of \cite{rifkin2002everything, Hui2021, kline2005revisiting, golik2013cross} demonstrate that kernel machines and deep neural networks trained using the squared loss, which is traditionally reserved for regression problems with continuous labels, can result in comparable classification performance to those trained with the more popular cross-entropy loss.

Motivated by these observations, recent work has sought to deepen theoretical understanding of the impact of the loss function in overparameterized classification tasks, starting with linear models. The hard-margin support vector machine (SVM) \cite{cortes1995support} is one popular classification algorithm which, given $n$ linearly separable training samples $(x_i, y_i) \in \R^d \times \{-1, 1\}$, aims to find a decision function $\fhat(x) = \ip{\thetahat}{x}$ by solving the convex program:
\[
\hat{\theta}  = \argmin_{\theta \in \R^d}\norm*{\theta}_2 ~~~\text{s.t. } y_i \theta^\top x_i \geq 1 \text{ for } i = 1, \dots, n.
\]
Equivalently, this solution maximizes the smallest distance from the hyperplane $\{x \colon \fhat(x)=0\}$ to any training example. It was first shown in~\cite{Muthukumar2021} that, in overparameterized settings, the SVM decision function can coincide exactly with the \emph{minimum norm interpolant} of the training data (i.e., that the learned $\hat{\theta} = \argmin_{\theta \in \R^d} \norm*{\theta}_2$ results in $\fhat(x_i) = y_i$ for all $i=1,\dots, n$). This phenomenon, termed \emph{support vector proliferation (SVP)}, has several important implications. First, since it is known that these two solutions can be obtained via gradient descent on the squared and logistic loss respectively, this gives theoretical backing to the empirical observation that different loss functions can lead to similar behavior in overparameterized settings \cite{rifkin2002everything, Hui2021}. This equivalence has also been exploited in several works (e.g., \cite{Muthukumar2021, Wang2021,wang2021multiclass}) to study the \textit{generalization} performance of the SVM by understanding the minimum-norm interpolant of the discrete labels, which can be more easily analyzed via a convenient closed-form solution. This approach can in fact be leveraged to prove classification consistency results for the SVM in settings which classical generalization bounds do not cover, as discussed in~\cite{Muthukumar2021, Hsu2021}. 

Despite considerable subsequent progress in our understanding of SVP~\cite{Hsu2021,ardeshir2021support,Wang2021,wang2021multiclass,Cao2021}, current characterizations crucially depend on strong distributional assumptions on the data or features. Specifically, conditions under which SVP occurs have only been given in settings where the features have Gaussian or independent sub-Gaussian entries.
Thus, they are essentially only applicable when a linear model is learned directly on input data which is itself high-dimensional. By contrast, in the more general setting of \emph{kernel SVM}, the training data $x \in \R^d$ is often first transformed (either explicitly or implicitly, using the kernel trick) via some feature map $\phi \colon \R^d \to \R^D$, and the SVM solution described above learns a function of the form $\fhat(x) = \ip{\hat{\theta}}{\phi(x)}$. For essentially any choice of kernel or non-linear feature map, the entries of $\phi(x_i)$ are \textit{not} independent. Indeed, in many common learning settings, $D \gg d$; i.e., training data of \textit{low or fixed} dimension (or effectively low dimension, e.g., from a low-dimensional manifold) are mapped to some higher- or even infinite-dimensional features which are intricately dependent on each other. As a case study, SVP was first empirically observed in the experiments of \cite{Muthukumar2021}, which used 1-dimensional data and a Fourier feature map whose entries have considerable structure. Existing theoretical results do not apply in such settings.

\subsection{Contributions}
In this paper, we study SVP for kernel methods (i.e., for general classes of feature maps) and provide precise non-asymptotic conditions under which it occurs. We consider a generative label model for classification of the form $\eta^*(x) = 2\P(y=1 \given x) - 1$, where $\eta^*$ is a function in an arbitrary RKHS and represents a natural affine transformation of the classification conditional probability. Our results are briefly summarized below.

\paragraph{SVP for general bounded features} Under the label model described above, we provide sufficient conditions on the spectrum of the kernel integral operator and training examples under which the minimum RKHS-norm interpolant of the training data is identical to the decision function learned by the kernel SVM.  In Section \ref{sec:assumps} we formally provide the main assumptions under which our result holds. At a high level, we only require that the leading kernel eigenfunctions comprise a \emph{bounded orthonormal system (BOS)}, and we leverage recent conditions from the literature for harmless interpolation of noise to ensure that the minimum-norm interpolant is close to an \emph{attenuated version}\footnote{Specifically, we define an operator $\scrSbr$ and require that $\etahat \approx \scrSbr \eta^*$ and that $\scrSbr$ is an $L_\infty$ contraction operator.} of $\eta^*$. 
Then, in Section \ref{sec:bosresults}, we characterize the settings in which SVP occurs under these assumptions. Compared to previous work, our requirements on the feature map and kernel spectrum are considerably weaker and allow us to show that SVP can occur in general RKHS settings without independence or sub-Gaussianity of features. Our proof techniques for these results depart significantly from prior literature that assumes worst-case labels~\cite{Muthukumar2021,Hsu2021,Wang2021,wang2021multiclass,Cao2021,ardeshir2021support}, and make novel conceptual connections between SVP and tools for sharp error analysis of interpolating methods.

\paragraph{Refined bounds for independent, sub-Gaussian features} In Section \ref{sec:subgresults}, we specialize and sharpen our results to the case where the features are independent and sub-Gaussian. Notably, unlike prior works, our use of a probabilistic generative label model and weaker spectral assumptions allow us to show that SVP can occur even when the minimum-norm interpolant generalizes well for the regression task. To the best of our knowledge, this is the first instance in which SVP has been established for such cases.

\paragraph{Generalization of the kernel SVM in new regimes} As in previous work on SVP, our results can be combined with analysis of the minimum-norm interpolant to analyze the generalization performance of the SVM. In Section \ref{sec:genresults}, we show that our analysis can be used to prove asymptotic classification consistency of the kernel SVM in new settings which are too (effectively) high-dimensional for classical generalization bounds to be informative. Compared to prior works on SVP which consider linear models on the input data, our results demonstrate that the kernel SVM can generalize well for a range of interesting kernels that correspond to more structured feature mappings.  

\subsection{Related Work}
We organize our discussion of related work under two verticals.
\paragraph{Support vector proliferation} The empirical observation that the squared and cross-entropy losses result in comparable classification performance in modern deep learning and kernel models \cite{rifkin2002everything, Hui2021, kline2005revisiting, golik2013cross} has led to increased theoretical interest in the phenomenon of support vector proliferation~\cite{Muthukumar2021, Hsu2021, ardeshir2021support, Cao2021, Wang2021, wang2021multiclass}, and more generally equivalences between the solutions obtained by various loss functions~\cite{lai2023general}. In particular, since the SVM and minimum-norm interpolant can be obtained by gradient descent on the cross-entropy (or similar exponential-tailed) \cite{soudry2018implicit, ji2019implicit, ji2020gradient, gunasekar2018implicit, ji2020directional} and squared losses, SVP provides an avenue for understanding the types of overparameterized settings in which this loss function equivalence can occur. 

Prior works provide conditions (typically in terms of the data covariance and/or the signal-to-noise ratio) under which the phenomenon of SVP occurs with high probability, for an arbitrary (i.e.~worst-case) choice of labels. However, these results all require strong assumptions on the data distribution and dimensionality. SVP was first shown under Gaussian design in \cite{Muthukumar2021}, and sufficient conditions were given in \cite{Hsu2021} for independent sub-Gaussian and Haar design, along with a converse result for the Gaussian case. Furthermore, \cite{ardeshir2021support} provided a converse result for SVP for anisotropic sub-Gaussian data and proved that a data dimension $d = \Theta(n \log{n})$ is needed for SVP in the isotropic case. We provide a more detailed contextualization between our upper bounds and the converse results of~\cite{ardeshir2021support} in Section~\ref{sec:bilevel}. The work of \cite{Wang2021, Cao2021,wang2021multiclass} develops similar results under Gaussian and sub-Gaussian mixture models for the binary, multiclass, and one-versus-all SVM settings. In our work, we focus on a particular generative model for binary labels and do not explicitly make sub-Gaussian tail assumptions on the covariate distribution in our main result. A second key difference with the existing literature is that, because of our probabilistic generative model on the labels, we obtain results under weaker random matrix conditions (and hence conditions on the spectrum of the kernel integral operator). 
These spectral conditions are similar to those used in analyses of minimum-norm interpolants (e.g., \cite{Bartlett2020, Muthukumar2020a, Tsigler2020, Mei2022, McRae2022}) and highlight a novel connection between SVP and harmless interpolation of noise. 

\paragraph{Generalization analysis of classifiers in overparameterized settings} 
Classical bounds on the generalization error of the SVM rely on sample compression (e.g., \cite{graepel2005pac, germain2011pac}) or Rademacher complexity (e.g., \cite{bartlett1999generalization, mcallester2003simplified}), which typically yield upper bounds on the test error in terms of the fraction of training samples which are support vectors or the margin achieved by the learned hyperplane on the training set, respectively. Moreover, the fraction of support vectors in the SVM was shown to be $o(1)$ in a range of settings where the number of training points is proportional to or much greater than the data dimension \cite{dietrich1999statistical, buhot2001robust, malzahn2005statistical}, implying good generalization. When the data or feature dimension is much larger than the number of training points, however, these results no longer apply. In fact, it has recently been shown that the SVM decision function can interpolate the training labels in highly overparameterized settings (i.e., that every point becomes a support vector) and that, as a result, traditional bounds become vacuous, in the sense that they are no better than a constant with high probability \cite{Muthukumar2021, Hsu2021}.

Recent analyses of benign overfitting in linear and kernel regression models have provided an alternative path to studying the behavior and generalization of overparameterized classifiers. These results have been exploited\footnote{We note that there is a wide body of work on asymptotic error analysis and for general $\ell_p$-regularizers, and accordingly restrict our discussion to the most related setting, i.e.~non-asymptotic error analyses of the SVM with $\ell_2$-regularization, here.} to provide error bounds in a signed Gaussian model \cite{Muthukumar2021}, high-dimensional linear discriminant analysis \cite{Chatterji2021}, and Gaussian and sub-Gaussian mixture models \cite{Wang2021, wang2021multiclass, Cao2021}. Most of these works (with the exception of~\cite{Chatterji2021}, which also requires very high-dimensional data) explicitly use SVP to link the performance of a classifier to that of an interpolating regression task which is easier to analyze. Hence, they only apply in the restrictive settings in which SVP has been shown to hold. 

By contrast, our analysis most closely resembles that of \cite{McRae2022}, who give a bound for the classification loss of the label interpolant obtained by minimizing the squared loss in a general RKHS setting. The analysis tools provided in~\cite{McRae2022} in turn have some precedent in the approach of~\cite{Hsu2014,Zhang2005} who considered explicitly regularized kernel regression under random design and minimal assumptions on the features. Notably, our results for bounded orthonormal systems are also stated in terms of spectral conditions on the kernel integral operator, rather than particular properties of the data distribution. The resultant error expressions are not as sharp as those obtained via Gaussian or sub-Gaussian assumptions, but they apply for more general choices of kernel and feature map. 

\section{Background and setup}\label{sec:setup}
Our main results build on the analytic framework for kernel regression and classification in \cite{McRae2022}. We summarize the main components of this framework here. Specifically, we introduce important concepts relating to the spectral properties of reproducing kernel Hilbert spaces and set up the supervised learning framework under which we will study the equivalence of the support vector machine (SVM) and minimum-norm interpolation (MNI) classification problems. 

\subsection{Notation}
For a probability space $(X, \mu)$ and $f \colon X \to \R$, we define the norms $\norm{f}_{L_p} \coloneqq  \parens*{ \E_{x \sim \mu} \abs{f(x)}^p }^{1/p}$ and $\norm{f}_{L_\infty} \coloneqq \text{ess sup}_x |f(x)|$. For $u \in \R^n$, $\norm{u}_{\ell_2}$ or $\norm{u}_{2}$ is the Euclidean norm. More generally, for a function $f$ in a Hilbert space $\scrH$, $\norm{f}_\scrH$ is the Hilbert norm. We define $L_2(X, \mu)$ to be the Hilbert space of square-integrable functions $f \colon X \to \R$ with inner product $\ip{f}{g}_{L_2} = \E_{x \sim \mu} f(x)g(x)$.
We denote the $\ell_2$ and $\scrH$ inner products by $\ip{\cdot}{\cdot}_{\ell_2}$ and $\ip{\cdot}{\cdot}_{\scrH}$, respectively. 
We will write the operator norm of an operator $T\colon \scrH_1 \to \scrH_2$ (for distinct Hilbert spaces $\scrH_1$, $\scrH_2$) with respect to the $\scrH_1$ and $\scrH_2$ norms as $\norm{T}_{\scrH_1 \to \scrH_2}  := \sup_{\norm{f}_{\scrH_1} \leq 1} \norm{Tf}_{\scrH_2}$, or, when $T: \scrH_1 \to \scrH_1$, simply as $\norm{T}_{\scrH_1}$. For a subspace $A$ of a Hilbert space $\scrH$, we denote the orthogonal projection operator onto $A$ as $\scrP_A$, and for an operator $T$ on $\scrH$, we denote $T_A := T\scrP_A$. Let $\diag^\perp(Z)$ denote the projection of a matrix $Z$ onto the set of matrices with zero diagonal. The identity operator is denoted by $\scrI$ (or $I$ when it is an ordinary matrix). The notation $e_i$ denotes the $i^\text{th}$ canonical basis vector in $\R^n$. Finally, the notation $a \lesssim b$ means $a \leq Cb$ for some universal constant $C>0$, $a \asymp b$ means $a \lesssim b$ and $b \lesssim a$, and $a \ll b$ means $a < cb$ for some (sufficiently small) constant $c>0$. 

\subsection{Reproducing kernel Hilbert spaces}
Let $\scrH$ be a reproducing kernel Hilbert space (RKHS) of real-valued functions on a probability space $(X, \mu)$. We denote the reproducing kernel of $\scrH$ as $k:X \times X \to \R$. By definition, $k$ satisfies the property that, for any $x \in X$ and $f \in \scrH$, $f(x) = \ip{k_x}{f}_\scrH$, and $k_x := k(x, \cdot)$ is the unique element in $\scrH$ with this property.

Our results rely crucially on the \textit{spectral} properties of the kernel used during the learning process. To this end, we first assume that samples of $X$ are drawn according to the probability measure $\mu$. Then, the kernel integral operator $\scrT: L_2(X, \mu) \to L_2(X, \mu)$ is given by
\[
\scrT f(\cdot) = \int_X k(\cdot, y)f(y)d\mu(y).
\]
Assuming $k$ is continuous, Mercer's theorem yields
\[
k(x, y) = \sum_{\ell=1}^{\infty} \lambda_\ell v_{\ell}(x) v_{\ell}(y),
\]
where $\{\lambda_\ell\}_{\ell=1}^\infty$ and $\{v_\ell\}_{\ell=1}^\infty$ are the eigenvalues and eigenfunctions of $\scrT$, respectively (a discussion of the relatively weak conditions under which such a decomposition holds can be found in \cite{Steinwart2012}). Furthermore, the eigenfunctions $\{v_\ell\}_{\ell=1}^\infty$ form an orthonormal basis for $L_2(X, \mu)$. Without loss of generality, we will assume the eigenvalues are sorted so that $\lambda_1 \geq \lambda_2 \geq \dots \geq 0$. For some index $p$, which we can select in our analysis, let $G := \text{span}(v_1, \dots, v_p)$ be the subspace of $L_2(X, \mu)$ corresponding to the largest $p$ eigenvalues, and let $\scrT_G$ be the restriction of the kernel integral operator to $G$.

\subsection{Classification setting}
Assume we observe a training set $\{x_i, y_i \}_{i=1}^n$, consisting of i.i.d. copies of $(x, y)$, where $x \sim \mu$ and $y \in \{-1, 1\}$. We assume the relationship between $x$ and $y$ is governed by the regression function
\begin{equation}\label{eq:classconditionalgenerative}
\eta^*(x) := \E[y \given x] = 2 \mathbf{P}(y=1 \given x) - 1
\end{equation}
for some $\eta^* \in G$.\footnote{The assumption that $\eta^*$ lies in the subspace $G$ is for convenience and can be relaxed at the expense of an extra term in our error expressions.} In other words, given $x$, $y$ equals $1$ with probability $\frac{\eta^*(x) + 1}{2}$ and equals $-1$ otherwise. Equation~\eqref{eq:classconditionalgenerative} constitutes a natural generative model for classification through performing regression on (an affine transformation of) the classification conditional probability; similar models appear in the literature on plug-in classifiers (e.g., \cite{Audibert2007}), albeit with slightly different regularity assumptions on $\eta^*$.
Accordingly, we consider classification procedures which yield an estimator $\etahat$ of $\eta^*$ and predict the label of a new point $x \in X$ via $\yhat = \sign(\etahat(x))$. It is often convenient to consider our binary observations as of the form $y_i = \eta^*(x_i) + \xi_i$, where $\xi_i := y_i - \eta^*(x_i)$ is interpreted as observation noise and satisfies $\E[\xi \given x] = 0$. We collect the observations $y_i$ and the noise $\xi_i$ into vectors $y$ and $\xi$, respectively; it will be clear from context whether these symbols refer to the scalar random variables or the vector of observed values.

Given a data set of the form described above, we let $\scrA : \scrH \to \R^n$ denote the sampling operator on $\scrH$, i.e., the operator such that $(\scrA f)_i = f(x_i)$. The adjoint $\scrA^*\colon \R^n \to \scrH$ of the sampling operator with respect to the $\scrH$ and $\ell_2$ inner products is characterized by $\scrA^*\beta = \sum_{i=1}^n \beta_i k_{x_i}$, for any $\beta \in \R^n$.  Finally, we denote the Gram matrix corresponding to the data set as $K := \scrA \scrA^* \in \R^{n\times n}$. The entries of this matrix are given by $K_{ij} = k(x_i, x_j)$. 

For our analysis, we also split the sampling operator into its components which act on $G$ and $G^\perp$, denoted $\scrA_G \coloneqq \scrC$ and  $\scrR \coloneqq \scrA_{G^\perp}$, respectively. We let $\scrC^* = \scrT_G^{-1}\scrA_G^*$ be the adjoint operator of $\scrA_G$ corresponding to the inner product $\ip{\cdot}{ \cdot}_{L_2}$. Finally, the \textit{excess classification risk} of an estimator $\hat{\eta}$ is defined as
\[
\scrE(\hat{\eta}) \coloneqq \P(\yhat \neq y) - \P(y \neq \sign (\eta^*)).
\]


\subsection{Ridgeless kernel regression}

The ridgeless kernel regression, or \textbf{minimum-norm interpolation (MNI)}, procedure solves the problem:
\begin{equation}
\begin{aligned}\label{mni}
    \etahat  &= \argmin_{\eta \in \scrH}~\norm*{\eta}_\scrH\\
    &\text{s.t. } \eta(x_i) = y_i \text{ for all } i = 1, \dots, n.
\end{aligned}
\end{equation}
This problem admits a closed form solution $\etahat(x) = \sum_{i=1}^n \betahat_i k(x, x_i)$, where $\betahat = K^{-1} y$. The properties of such estimators have been studied extensively in several recent works (e.g., \cite{Liang2020, Bartlett2020, Hastie2022, Mei2022, McRae2022}). At a high level, most of these works decompose the learned estimate as 
\begin{equation}\label{eq:mni_decomp}
\begin{aligned}
\etahat = \scrA^*(\scrA \scrA^*)^{-1} y &= \underbrace{\scrA^*(\scrA \scrA^*)^{-1} \scrA \eta^*}_{\eqqcolon \etahat_0} + \underbrace{\scrA^*(\scrA \scrA^*)^{-1} \xi}_{\eqqcolon \epsilon}\\
\end{aligned}
\end{equation}
and then, under various assumptions, characterize when (\emph{i}) the first term approximates $\eta^*$ and (\emph{ii}) the second term is negligible. Most central to our analysis, \cite{McRae2022} show that, under certain spectral conditions on the kernel, the estimate satisfies $\etahat \approx \scrSbr \eta^*$, where $\scrSbr = n \scrT_G \parens*{\alphabr \scrI_G + n\scrT_G}^{-1}$ is the \textit{idealized survival operator},\footnote{We will formally define $\alphabr$ in Assumption \ref{assumps}.} which represents how well the true function is recovered by MNI. We will see below that SVP can occur in a variety of settings where $\scrSbr$ uniformly attenuates $\eta^*$ in $L_\infty$ norm.

\subsection{Kernel SVM}
The \textbf{hard-margin kernel support vector machine (SVM)} procedure aims to find a linear classifier that maximizes the training data margin in the space $\scrH$. Particularly, the primal and dual forms of this problem can be written as
\begin{equation}
\begin{aligned}[c]\label{svm}
    \etahat  &= \argmin_{\eta \in \scrH}   \norm*{\eta}_\scrH\\
    &\text{s.t. } y_i \eta(x_i) \geq 1 \text{ for } i = 1, \dots, n.\\
\end{aligned}
\qquad \overset{\makebox[0pt]{\mbox{\normalfont\tiny\sffamily}}}{\iff} \qquad
\begin{aligned}[c]
        \betahat &= \argmax_{\beta \in \mathbb{R}^n}  y^\top \beta - \frac{1}{2}\beta^T K \beta\\
        &\text{s.t. } y_i \beta_i \geq 0 \text{ for } i=1, \dots, n,
\end{aligned}
\end{equation}
with $\etahat= \sum_{i=1}^n \betahat_i k_{x_i}$. Observing the dual form, we note here that the estimators $\etahat$ obtained by the MNI and SVM learning procedures both have a representation as a linear combination (with coefficients given in the vector $\betahat$) of the input features $k_{x_i}$, for $i = 1, \dots, n$. As was originally noted in the discussion of \cite{Muthukumar2021, Hsu2021}, it is readily verified that the choice $\betahat = K^{-1}y$ (which characterizes the solution to \eqref{mni}) also maximizes the dual objective function of \eqref{svm}; however, this choice may not satisfy all of the additional constraints of the SVM. Our main results in the following section aim to characterize when the estimators $\etahat$ obtained by solving these two problems coincide \textit{exactly}. Since the \textit{support vectors} are the training points which satisfy $\etahat(x_i) = \pm 1$ (and hence for which the primal constraints are active), this equivalence implies every training point is a support vector in the learned SVM. 

\section{Main results}\label{sec:results}
In this section, we provide our main results characterizing SVP in the RKHS setting described above. After describing our main assumptions and defining the key quantities which appear in our results, in Section~\ref{sec:bosresults} we state a theorem for general bounded feature families and provide an example of a bi-level feature design which exhibits SVP. Then, in Section~\ref{sec:subgresults} we use our analysis framework to strengthen our results in the case of independent, sub-Gaussian features, hence showing that SVP can occur in a strictly expanded regime than that provided by the main results of~\cite{Hsu2021} under our generative label model. Finally, we demonstrate that our result immediately implies good generalization (in the sense of statistical consistency) of the kernel SVM in new overparameterized settings. 
\subsection{Feature model and sampling conditions}\label{sec:assumps} We consider kernels for which the eigenfunctions of the kernel integral operator $v_{\ell}(x)$ satisfy one of the following two properties:
\begin{enumerate}
    \item \textbf{Bounded orthonormal system on $G$ (BOS):} The eigenfunctions of $\scrT_G$ satisfy $\sum_{\ell=1}^p v_{\ell}^2(x) \leq Cp$ with probability $1$ over $x$ for some constant $C \geq 1$.
    \item \textbf{Independent sub-gaussian (SG):} For $x\sim\mu$, the $\{v_{\ell}(x)\}_{\ell\geq 1}$ are independent, zero-mean, sub-Gaussian random variables.
\end{enumerate}

The BOS condition on $G$ is satisfied by several common feature families, such as Fourier features, certain orthonormal polynomial families, discrete orthonormal systems, etc. Notably, in many of these cases, the features are far from independent. 

Our analysis considers the setting where the sampling locations in the training set are drawn independently from the distribution $\mu$ and satisfy certain properties with high probability (which can be separately verified via concentration bounds stated in Appendix \ref{app:bounds}) The first two properties are equivalent to those in Theorem 1 of \cite{McRae2022}. 
The last condition is a property of our analysis and is satisfied for a variety of choices of features and $\eta^*$.
After listing our main assumption below, we provide a brief interpretation and discussion of when it is satisfied.
\begin{assumption}\label{assumps}
For some $0<\delta<1/2$, the training samples $x_1, \dots, x_n \simiid \mu $ satisfy all of the following with probability at least $1-\delta$:
\begin{enumerate}
    \item (Residual Gram matrix concentration) There exist constants $\alpha_L \asymp \alpha_U > 0$ for which the residual Gram matrix satisfies
    $$\alpha_L I_n \preccurlyeq \scrR \scrR^* \preccurlyeq \alpha_U I_n.$$ 

    \item (Feature covariance concentration on $G$) Let $\scrC_{\setminus i}$ be the sampling operator on $G$ for the leave-one-out data set $\{x_1, \dots, x_{i-1}, x_{i+1}, \dots, x_n\}$ with $L_2$-adjoint $\scrC_{\setminus i}^*$. Then, for all $i \in \{1, \dots, n\}$, the following holds  for some $c \leq \frac{1}{2}$:
    $$\frac{\alpha_U - \alpha_L}{\alpha_U + \alpha_L} + \frac{2}{n}\norm{\scrC_{\setminus i}^* \scrC_{\setminus i} - n \scrI_G}_{L_2} \leq c.$$
    \item ($L_\infty$ control on expected recovered signal) Let $\alphabr$ be the harmonic mean of $\alpha_L$ and $\alpha_U$. Let $\scrSbr \coloneqq n \scrT_G \parens*{\alphabr \scrI_G + n\scrT_G}^{-1}$ be the idealized survival operator and $\scrBbr := \scrI_G - \scrSbr$ be the idealized bias operator. Then
$$b \coloneqq \norm{\scrBbr \eta^*}_\infty  \lesssim 1 - \norm{ \scrSbr \eta^*}_{\infty}.$$

\end{enumerate}
\end{assumption}

The first condition of residual Gram matrix concentration is needed to ensure harmless interpolation of noise in MNI and has been shown to hold with high probability in many situations, such as for sub-Gaussian features \cite{Bartlett2020, Tsigler2020} and in the case where the leading eigenspace satisfies a hypercontractivity condition \cite{Mei2022}. We provide relevant concentration bounds in Appendix \ref{app:bounds} which can be used to verify this condition in both the BOS and SG cases.
Note that unlike previous results on SVP which are label-agnostic and depend crucially on concentration of the \textit{entire} Gram matrix to a multiple of identity (including \cite{Muthukumar2021, Hsu2021}), we only require this for the residual component.

The second condition requires covariance concentration on $G$. The full-sample version of this assumption is standard in regression. Our leave-one-out version holds with high probability for BOS features when $n \gtrsim p \log{n}$ (see, e.g., Lemma \ref{lem:C_conc_BOS}) and for SG features when $n \gtrsim p$ (Lemma \ref{lem:C_conc_ind}).

The third condition requires that the idealized survival operator $\scrSbr$ shrinks $\eta^*$ (approximately) \emph{uniformly}.
For example, it is satisfied in any of the following cases:
\begin{itemize}
    \item \textbf{Identical leading eigenvalues:} The leading $p$ eigenvalues of $\scrT$ are identically equal to some $\lambda >0$. Then, $\scrSbr = \frac{n\lambda}{\alphabr + n\lambda} \scrI_G$ and $\scrBbr = \frac{\alphabr}{\alphabr + n\lambda} \scrI_G$, and the fact that $\norm{\eta^*}_\infty \leq 1$ implies that the assumption holds (with the constant in the $\lesssim$ being exactly 1). We explore this scenario in Section \ref{sec:bilevel}.
    \item \textbf{1-sparse target function:} We obtain the same result without identical leading eigenvalues if $\eta^*$ is in the span of a single leading eigenvector, i.e., $\eta^* = \gamma v_{\ell}$ for some $1 \leq \ell \leq p$ and constant $\gamma$.
    \item \textbf{Extreme shrinkage:} In general, if $\alphabr \gg n\lambda_1$, then $\scrSbr \eta^*$ will shrink towards $0$ while $\scrBbr \eta^*$ approaches $\eta^*$ (and sufficient shrinkage coupled with bounded $v_\ell$ yields the desired $L_\infty$ control). Note the similarity to the SVP condition $\sum_{i=1}^n \lambda_i \gg n \lambda_1$ from \cite{Hsu2021}.
\end{itemize}

\subsection{General conditions for SVM Equivalence}\label{sec:bosresults}
Our main result provides sufficient conditions under Assumption \ref{assumps} for every training point to be a support vector, and hence for MNI and SVM to yield identical solutions. The proof of this result is given in Section \ref{sec:proofs}.

\begin{theorem}[SVP in Bounded Orthonormal Systems]\label{thm:bos}
Consider the classification setting described in Section \ref{sec:setup} with BOS features. Suppose Assumption \ref{assumps} is satisfied for $\eta^* \in G$. Then, with probability at least $1-\delta - \frac{2}{n}$, the solutions of the MNI procedure (\ref{mni}) and the SVM procedure (\ref{svm}) coincide if

\begin{equation}
\begin{aligned}
        \frac{1}{b}\sqrt{\frac{p\log{n}}{n}} &+ \frac{\alpha_U - \alpha_L}{\alpha_U + \alpha_L} \parens*{\frac{\sqrt{\log{n}}}{b} + \sqrt{n}} + c^2\sqrt{p} \ll 1.
\end{aligned}
\end{equation}
\end{theorem}
Our result depends in a nuanced manner on the idealized bias term $b$, the number of samples $n$, the dimension $p$ of $G$, the concentration of our residual Gram matrix $\scrR \scrR^*$, and (through the constant $c$) the sample covariance concentration on $G$.
This result shows that SVP can occur more generally in kernel settings without sub-Gaussian or independent features (including the Fourier feature setting in \cite{Muthukumar2021} where this phenomenon was originally observed).
Note that, unlike the results of \cite{Muthukumar2021, Hsu2021}, our analysis assumes a particular model for the labels and does \textit{not} allow for adversarial label selection. However, our results apply under much weaker assumptions on the features, as we discuss in the next section.

\subsubsection{Implications for the bi-level ensemble}\label{sec:bilevel}
As a concrete example, we now introduce the overparameterized \textit{bi-level ensemble} as defined in \cite{Muthukumar2021, McRae2022}. This example is useful for understanding the types of settings in which SVP can occur and is parameterized by a tuple $(n, \beta, q, r)$, where $\beta > 1$, $0 < r < 1$, and $q \leq \beta - r$. Here, $d = n^\beta$ is the number of features used for training, the first $p = n^r$ features correspond to larger eigenvalues, and 
\begin{equation}
    \label{eq:bilevel}
    \lambda_{\ell} = 
    \begin{cases}
        1, & \text{for } 1 \leq \ell \leq p \\
    n^{-\beta+r+q}, & \text{for } p < \ell \leq d.    \end{cases}
\end{equation}
Note that $q \leq \beta - r$ (with equality representing the isotropic case) is required for correct eigenvalue ordering and $r<1$ ensures $p<n$. Here, $q$ can be interpreted as controlling the ``effective overparameterization'', i.e., the separation between the leading and residual eigenvalues. This type of featurization was inspired by spiked covariance models \cite{Muthukumar2021} and has the useful property that it permits harmless interpolation of noise in linear regression throughout its permissible range of parameters.
Given the close connection of our analysis to that of harmless interpolation of the MNI, this ensemble is a natural setting in which to explore the phenomenon of SVP. Below we specialize the result of Theorem \ref{thm:bos} to the case of bounded features under the bi-level ensemble. The proof is provided in Appendix \ref{app:bilevel_proof}.

\begin{corollary}[SVP in BOS: bi-level ensemble]\label{cor:bos}
Consider the bi-level ensemble with parameters $(n, \beta, q, r)$ and bounded features $\{v_{\ell}(x)\}_{\ell=1}^d$. 
Suppose $\eta^* \in G$ in the classification setting of Section \ref{sec:setup} and $\sup_{x} \left|\sum_{\ell >p}(|v_{\ell}(x)|^2 - 1)\right| \lesssim n^{\frac{\beta}{2}+1}$. 
Then, as $n \to \infty$, SVP occurs with probability $\to 1$ if $\beta > 3$, $q > \frac{1-r}{2}$, and  $r < \frac{2}{3}$.
\end{corollary}

This corollary demonstrates that SVP occurs in this setting as long as the degree of effective overparameterization (controlled by $\beta$ and $q$) is sufficiently high. Notably, this permits situations where the features are highly structured. In fact, we only assume the $v_\ell$ are bounded and satisfy the thin-shell type condition $\sup_{x} \left|\sum_{\ell >p}(|v_{\ell}(x)|^2 - 1)\right| \ll n^{\frac{\beta}{2}+1}$, which ensures that $k(x, x)$ does not vary too much and is necessary for concentration of the residual Gram matrix (similar requirements are given in related works which study this concentration phenomenon such as \cite{Mei2022}, \cite{Bartl2022}) This can be easily verified to hold for many structured feature families, including Fourier features, for which $|v_\ell(x)|^2 = 1$ uniformly for all $x$. We further note that, although we state asymptotic results under the bi-level model for clarity of exposition, Theorem \ref{thm:bos} (and the proof in Appendix \ref{app:bilevel_proof}) allows us to obtain finite-sample guarantees that hold with high probability in a wide range of settings with non-uniform eigenvalues.

\subsection{Improvement in the independent, sub-Gaussian case}\label{sec:subgresults}

We now consider the case in which the features $\{v_\ell(x)\}_\ell$ are independent and sub-Gaussian random variables. This is the setting which was considered in earlier works on support vector proliferation \cite{Hsu2021, Muthukumar2021}. These works generally do not assume a particular label model and state conditions in terms of certain notions of ``effective rank'' of the entire feature covariance. To compare to this line of work, we also state results in the independent, sub-Gaussian case. Our results apply to a more specific label model which implicitly imposes an additional distributional assumption compared to previous work, since $\eta^*(x)$, (which is a linear combination of $\{v_i(x)\}_{i=1}^p$) must be uniformly bounded in our model. However, our analysis framework allows us to recover many existing results while also showing that SVP can occur in a broader range of settings. The proof of this theorem is given in Section \ref{sec:subg_proofs}.

\begin{theorem}[SVP for independent, sub-Gaussian  features]\label{thm:subg}
Consider the classification setting described in Section \ref{sec:setup}. Suppose that Assumption \ref{assumps} is satisfied for $\eta^* \in G$ and that the the features $\{v_{\ell}(x)\}_{\ell>1}$ are zero-mean, independent, and sub-Gaussian. Then, with probability at least $1-\delta - ne^{-p}-\frac{3}{n}$, the solutions of the MNI procedure (\ref{mni}) and the SVM procedure (\ref{svm}) coincide if
\begin{equation}
\begin{aligned}
     \frac{1}{b}\sqrt{\frac{p\log{n}}{n}} + \frac{\alpha_U - \alpha_L}{\alpha_U + \alpha_L} \frac{\sqrt{\log{n}}}{b} &+ \sqrt{\frac{\lambda_{p+1} n \log{n}}{\alpha_L}} + c\sqrt{\log{n}} \ll 1.
\end{aligned}
\end{equation}
\end{theorem}

We can see that the last two terms are much sharper than the corresponding terms in the BOS case (e.g., $c\sqrt{\log n}$ instead of $c^2\sqrt{p}$). This improved bound, coupled with the faster concentration rates of the operators $\scrR \scrR^*$ and $\scrC^* \scrC$ for sub-Gaussian features, can be shown to weaken the conditions for SVP under the bi-level ensemble. We state this condition in the following corollary, proved in Appendix \ref{app:bilevel_proof}.

\begin{corollary}[SVP for independent, sub-Gaussian case: bi-level ensemble]\label{cor:subg}
Consider the bi-level ensemble with parameters $(n, \beta, q, r)$ and independent, zero-mean, sub-Gaussian features $\{v_{\ell}(x)\}_{\ell=1}^d$. Suppose $\eta^* \in G$ in the classification setting of Section \ref{sec:setup}. Then, as $n \to \infty$, SVP occurs with probability $\to 1$ if $\beta > \max{\parens{1, 3-2r-2q}}$ and $q > \frac{1-r}{2}$.
\end{corollary}

This corollary is consistent with the results of \cite{Hsu2021}, which show that SVP occurs in the bi-level model for $q > 1-r$ and $\beta > 1$ for arbitrary labels $y_i$. We recover this result under our label model while also finding that SVP can occur in the range $\frac{1-r}{2} < q < 1-r$ as long as $\beta$ is sufficiently large (i.e., as long as there is sufficient overparameterization). Unlike previous settings in which SVP was shown to occur, this is a regime in which the MNI procedure turns out to be regression-consistent, that is, it can recover $\eta^*$ exactly in the limit as $n \to \infty$ \cite{Muthukumar2021}\footnote{We note here that~\cite{ardeshir2021support} showed that SVP is \emph{unlikely} to happen (on sub-Gaussian data) for the regime where $d_2 := \frac{(\sum_{\ell \geq 1} \lambda_j)^2}{\sum_{\ell \geq 1} \lambda_j^2} \lesssim n \log n$ and $d_\infty^2 := \frac{(\sum_{j \geq 1} \lambda_j)^2}{\lambda_1^2} \gtrsim d_2 n$. However, this regime is not captured by any parameterization of the bilevel ensemble as we always have $d_\infty^2 \ll d_2 n$. Our results apply to a different regime in which we allow $d_{\infty} \leq n$ and still show that SVP can hold.}.

\subsection{Generalization of the kernel SVM in new settings}\label{sec:genresults}
In settings where support vector proliferation occurs, the equivalence between the MNI and the kernel SVM provides an approach for characterizing the generalization error of the SVM classifier. Specifically, we can use the MNI, which has a closed-form solution and has been analyzed extensively in recent literature, to understand the generalization of the kernel SVM. Indeed, as noted in \cite{Hsu2021}, traditional SVM generalization bounds are often based on the fraction of support vectors and hence become uninformative in the types of overparameterized settings where SVP occurs (and this fraction approaches $1$). Nevertheless, the equivalence to MNI can be used to demonstrate that the kernel SVM can indeed generalize well (in the sense of asymptotic consistency) even in highly overparameterized settings where \textit{every} training point is a support vector. Previously, such results were shown for linear models on Gaussian and independent-sub-Gaussian data~\cite{Muthukumar2021,Hsu2021,Wang2021,wang2021multiclass,Cao2021}.
We now demonstrate that such results can also be shown for general kernel methods under the bi-level ensemble, addressing in part an open question raised in~\cite{Muthukumar2021}. Recall that under this ensemble, classical margin bounds can be used to show that the kernel SVM is classification-consistent for $q<1-r$. By leveraging generalization results for the MNI from \cite{McRae2022}, we show below that SVP can be used to show that the kernel SVM is also consistent in some cases where $q>1-r$, i.e. where there is a high degree of overparameterization.

\begin{corollary}[Generalization of kernel SVM in the bi-level ensemble]\label{cor:gen}
Consider the bi-level ensemble with parameters $(n, \beta, q, r)$. Suppose $\eta^* \in G$ in the classification setting of Section \ref{sec:setup}. Then, 
\begin{enumerate}
    \item If $\{v_{\ell}(x)\}_{\ell=1}^d$ satisfy the BOS condition and the thin-shell condition of Corollary \ref{cor:bos}, the excess classification risk of the kernel SVM goes to $0$ with probability $\to 1$ as $n \to \infty$ if 
    \begin{equation*}
        \beta > 3, r<\frac{2}{3}, \text{ and } 0 \leq q < \frac{3}{2}(1-r).
    \end{equation*}
    \item If $\{v_{\ell}(x)\}_{\ell=1}^d$ are zero-mean, independent, and sub-Gaussian, the excess classification risk of the kernel SVM goes to $0$ with probability $\to 1$ as $n \to \infty$ if 
    $$0\leq q < \min{\parens*{1-r+\frac{\beta-1}{2}, \frac{3}{2}(1-r)}}.$$
\end{enumerate}
\end{corollary}

We note that the second part of Corollary \ref{cor:gen} is very similar to the condition implied by the works \cite{Muthukumar2021, Hsu2021}, while requiring a slightly stronger condition on $q$. However, these prior results assume that the true function is 1-sparse and linear; by contrast, we consider a flexible classification setting where the target function $\eta^*$ is assumed to belong to the subspace $G$ of $\scrH$, and the $y_i$ follow the label model described in Section \ref{sec:setup}. Furthermore, the first part of Corollary \ref{cor:gen} applies to a broad range of non-independent feature families with possibly heavier-tailed distributions. 

\section{Proofs of main theorems} \label{sec:proofs}
In this section, we provide proofs of our main results. We begin by describing our overall proof technique, before considering the main components of the analysis separately, deferring the proofs of useful technical lemmas to Appendix \ref{app:bounds}.
\subsection{Overview of analysis}
Our results build off of the deterministic conditions for support vector proliferation provided in \cite{Hsu2021}. For self-containment, we first restate this result for the kernel setting with the notation introduced in Section \ref{sec:setup}:

\begin{proposition}[Lemma 1 in \cite{Hsu2021}]\label{prop:conditions}
For a kernel $k(x, y)$ corresponding to RKHS $\scrH$,
let the minimum-norm label interpolant $\etahat$ be given by the formula
\[
	\etahat(x) = \sum_{i=1}^n z_i k(x, x_i),
\]
where $z = K^{-1} y$, and $K = [k(x_i, x_j)]_{ij}$ is the kernel Gram matrix, assumed to be non-singular. Furthermore, let $\etahat_{\setminus i}$ be the minimum-norm label interpolant constructed using the leave-one-out data set consisting of points $\{(x_j, y_j)\}_{j=1, j\neq i}^n$. 

Then, the following conditions are equivalent:
\begin{enumerate}
	\item Every data point $(x_i, y_i)$ is a support vector for the max-margin SVM, i.e., the SVM boundary and the binary label interpolant are identical.
	\item For every $i \in \{1, \dots, n\}$, $z_i y_i > 0$, i.e., $z_i \neq 0$ and $\sign(z_i) = y_i$.
	\item For every $i \in \{1, \dots, n\}$, $y_i \etahat_{\setminus i}(x_i) < 1$.
\end{enumerate}
\end{proposition}

Our analysis aims to give conditions under which $|\etahat_{\setminus i}(x_i)| < 1$ for all $i$. When this occurs, Proposition \ref{prop:conditions} immediately implies that all training points are support vectors. Recalling the decomposition (\ref{eq:mni_decomp}), we can write
\begin{equation*}
 |\etahat_{\setminus i}(x_i) | \leq |\epsilon_{\setminus i}(x_i)| + |\scrP_{G^\perp} \etahat_{\setminus i, 0}(x_i)| + |\scrP_{G}\etahat_{\setminus i, 0}(x_i) - \scrSbr\eta^*(x_i)| + |\scrSbr \eta^*(x_i)|.
\end{equation*}
By Assumption \ref{assumps}, $1 - |\scrSbr \eta^*(x_i)| \gtrsim b$, so the above bound will be strictly less than 1 when
\begin{equation}\label{eq:svm-cond}
\begin{aligned}
\underbrace{\frac{|\epsilon_{\setminus i}(x_i)|}{b}}_{A} &+ \underbrace{\frac{|\scrP_{G^\perp} \etahat_{\setminus i, 0}(x_i)|}{b}}_{B} + \underbrace{\frac{|\scrP_{G}\etahat_{\setminus i, 0}(x_i) - \scrSbr\eta^*(x_i)|}{b}}_{C}
\lesssim 1
 \end{aligned}
\end{equation}
for all $i$. Term $A$ represents the interpolation of label noise. Terms $B$ and $C$ represent how well the noiseless estimate $\etahat_{\setminus i, 0}$ is approximated at $x_i$ by the idealized survival $\scrSbr \eta^*$: $B$ represents the component orthogonal to $G$, and $C$ is the component in $G$.
In the following three subsections, which focus on the BOS feature case, we bound these three terms separately, yielding general conditions under which the condition \eqref{eq:svm-cond} is satisfied (Theorem \ref{thm:bos}). Then, in Section \ref{sec:subg_proofs}, we consider the case of independent, sub-Gaussian features and obtain sharper bounds for terms $B$ and $C$ (yielding Theorem \ref{thm:subg}).

\subsection{Noise error}
To bound term $A$ in \eqref{eq:svm-cond}, first observe that the noise variables $\xi$ are bounded with $|\xi_i| \leq 2$ for all $i\leq n$. Hence, the vector $\xi$ is sub-Gaussian in the following sense:
for any fixed vector $z \in \R^n$, $\ip{z}{\xi}$ is a sub-Gaussian random variable with sub-Gaussian norm of $\norm{z}_2$ within a constant.
This means that, for any fixed $z \in \R^n$, with probability at least $1 - \delta'$ with respect to $\xi$,
$\abs{ \ip{z}{\xi} } \lesssim \norm{z}_2 \sqrt{\log \delta'^{-1}}$. So, since 
\[
    \abs{\epsilon(x)} = \abs{\ip{k_x}{\scrA^*(\scrA \scrA^*)^{-1}\xi}_\scrH} = \abs{\ip{(\scrA \scrA^*)^{-1} \scrA k_x}{\xi}},
\]
it suffices to bound $\norm{(\scrA \scrA^*)^{-1} \scrA k_x}_{2}$. We leverage this observation in the following lemma.

\begin{lemma}[Upper bound on noise term]\label{lem:noise_term}
Let $x \in X$ be independent of the training data, and suppose Assumption \ref{assumps} holds. Then, with probability at least $1-\frac{1}{n^2}$,
\[
    \abs{\epsilon(x)} \lesssim \sqrt{\frac{p \log{n}}{n}} + \frac{1}{\alphatl} \norm{\scrR k_x^R}_2\sqrt{\log{n}},
\]
where $\alphatl = \frac{\alpha_L + \alpha_U}{2}$ and $k_x^R = \scrP_{G^{\perp}} k_x$.
\end{lemma}
\begin{proof}[Proof of Lemma \ref{lem:noise_term}]
Let $x \in X$ be independent of the data. We have
\[
    \epsilon(x) = \ip{k_x}{\scrA^*(\scrA \scrA^*)^{-1}\xi}_\scrH = \ip{(\scrA \scrA^*)^{-1} \scrA k_x}{\xi}.
\]
With respect to the randomness in $\xi$, this is a sub-Gaussian random variable whose sub-Gaussian norm is within a constant of $\norm{(\scrA \scrA^*)^{-1} \scrA k_x}_2$.
By Lemma 8 of \cite{McRae2022}, we have
\begin{equation*}
\begin{aligned}
    \norm{(\scrA \scrA^*)^{-1} \scrA k_x}_2 &= \norm{(\scrA \scrA^*)^{-1} (\alphatl I_n + \scrA_G \scrA_G^*)(\alphatl I_n + \scrA_G \scrA_G^*)^{-1}\scrA k_x}_2\\
    &\leq \frac{1}{2}\parens*{\frac{\alpha_U}{\alpha_L} + 1} \norm{(\alphatl I_n + \scrA_G \scrA_G^*)^{-1}\scrA k_x}_2\\
    &\lesssim \norm{(\alphatl I_n + \scrA_G \scrA_G^*)^{-1}\scrA_G k_x^G}_2 + \norm{(\alphatl I_n + \scrA_G \scrA_G^*)^{-1}\scrR k_x^R}_2,
    \end{aligned}
\end{equation*}
where $\alphatl = \frac{\alpha_L + \alpha_U}{2}$ and $k_x^G$ and $k_x^R$ denote $\scrP_G k_x$ and $\scrP_{G^{\perp}} k_x$, respectively. For the first of these terms, we have 
\begin{equation*}
    \begin{aligned}
    \norm{ (\alphatl I_n + \scrA_G\scrA_G^*)^{-1}\scrA_G k_x^G}_2
    &= \norm{\scrC (\alphatl \scrI_G + \scrT_G \scrC^* \scrC)^{-1} k_x^G}_2\\
    &=  \norm{\scrC (\alphatl \scrT_G^{-1} + \scrC^* \scrC)^{-1}  \scrT_G^{-1} k_x^G}_2\\
    &\leq \norm{\scrC}_{L_2 \to \ell_2} \norm{ (\alphatl \scrT_G^{-1} + \scrC^* \scrC)^{-1}}_{L_2 \to L_2} \norm{ \scrT_G^{-1}  k_x^G }_{L_2}\\
    &\lesssim \sqrt{n}\parens*{\frac{1}{n}} \sqrt{p} = \sqrt{\frac{p}{n}}.
    \end{aligned}
\end{equation*}
Here, the last inequality holds from the assumption on the concentration of $\scrC^* \scrC$ and the fact that $\norm{ \scrT_G^{-1}  k_x^G }_{L_2}^2 = \norm{\sum_{\ell=1}^p v_{\ell}(x) v_\ell}_{L_2}^2 = \sum_{\ell=1}^p v_{\ell}^2(x) \lesssim p$ by Parseval's identity and the BOS assumption. 

For the second term, we have 
\begin{equation*}
    \norm{(\alphatl I_n + \scrA_G \scrA_G^*)^{-1}\scrR k_x^R}_2 \leq \norm{(\alphatl I_n + \scrA_G \scrA_G^*)^{-1}} \norm{\scrR k_x^R}_2
    \leq \frac{1}{\alphatl} \norm{\scrR k_x^R}_2
\end{equation*}
Applying a sub-Gaussian tail bound yields the result. 

We also note here that for independent, sub-Gaussian features, the above argument also applies, except we use the fact that 
\[
    \norm{\scrT_G^{-1} k_x^G}_{L_2} = \parens*{\sum_{\ell = 1}^p v_\ell^2(x)}^{1/2} \lesssim \sqrt{p}
\]
with probability at least $1-e^{-p}$ (see, e.g. \cite[Theorem 3.1.1]{Vershynin2018}) (note that this inequality holds deterministically in the BOS case). So, the final result holds with probability at least $1-e^{-p}-\frac{1}{n^2}$, rather than $1-\frac{1}{n^2}$. We will use this fact in Section \ref{sec:subg_proofs}.
\end{proof}


Now, we apply this lemma to the leave-one-out training set corresponding to index $i$ and let $x = x_i$. In this case, $\norm{\scrR_{\setminus i} k_{x_i}^R}_2 =\norm{\diag^{\perp}(\scrR \scrR^*)e_i}_2 \leq \norm{(\scrR \scrR^* - \alpha_L I_n)e_i}_2 \leq \alpha_U - \alpha_L$. Recalling the definition of $\alphatl$, this yields
\begin{equation*}
    |\epsilon_{\setminus i}(x_i)| \lesssim \sqrt{\frac{p \log{n}}{n}} + \frac{\alpha_U - \alpha_L}{\alpha_U + \alpha_L} \sqrt{\log{n}}
\end{equation*}
with probability at least $1-\frac{1}{n^2}$. We can then take a union bound over the $n$ leave-one-out sets and divide by $b$ to obtain the final bound on term $A$: with probability at least $1-\frac{1}{n}$, for all $i =1, \cdots, n$, 
\begin{equation}\label{eq:termA}
A \lesssim \frac{1}{b} \cdot \sqrt{\frac{p \log{n}}{n}} + \frac{1}{b} \cdot \frac{\alpha_U - \alpha_L}{\alpha_U + \alpha_L} \sqrt{\log{n}}.
\end{equation}

\subsection{Bias error due to $G^\perp$}
For term $B$ in \eqref{eq:svm-cond},
\begin{equation*}
| \scrP_{G^{\perp}}\etahat_{\setminus i, 0}(x_i)|
    = \abs*{\sum_{j = 1, j\neq i}^n (K_{\setminus i}^{-1}\scrA_{\setminus i} \eta^*)_j k^{R}(x_j, x_i)} \leq \norm*{K_{\setminus i}^{-1}\scrA_{\setminus i} \eta^*}_2 \norm*{\diag^{\perp}(\scrR \scrR^*)e_i}_2
\end{equation*}
by the Cauchy-Schwarz inequality. The following lemma bounds the first of these terms.
\begin{lemma}[Useful bound for $G^\perp$ term]\label{lem:g_perp_term}
Under Assumption \ref{assumps} and the classification model described in Section \ref{sec:setup}, 
$
\norm{K^{-1}\scrA \eta^*}_2 \lesssim \frac{b\sqrt{n}}{\alpha_L}.
$
\end{lemma}
\begin{proof}[Proof of Lemma \ref{lem:g_perp_term}]
Recall the decomposition $K = \scrA_G \scrA_G^* + \scrR \scrR^*$. Then, by the push-through identity,
\begin{equation*}
\begin{aligned}
    \norm{(\scrR \scrR^* + \scrA_G\scrA_G)^{-1}\scrA\eta^*}_2 &= \norm{ (\scrR\scrR^*)^{-1}\scrA_G (\scrI_{G} + \scrA_G^*(\scrR\scrR^*)^{-1}\scrA_G)^{-1} \eta^*}_2\\
    &= \norm{(\scrR\scrR^*)^{-1}\scrA_G \scrB \eta^*}_2\\
    &\leq \norm{(\scrR\scrR^*)^{-1}}_{\ell_2 \to \ell_2} \norm{\scrA_G }_{L_2 \to \ell_2} \norm{\scrB\eta^* }_{L_2}\\
    &\lesssim \frac{1}{\alpha_L} \sqrt{n}   b\parens*{1 + \frac{c}{1-c}} \lesssim \frac{b\sqrt{n}}{\alpha_L}
    \end{aligned}
\end{equation*}

In the last line, we use the fact that $\norm{\scrA_G }_{L_2 \to \ell_2} = \sqrt{\norm{\scrC^* \scrC}_{L_2}} \asymp \sqrt{n}$ by the second part of Assumption \ref{assumps}. Furthermore, we have denoted $\scrB := (\scrI_{G} + \scrA_G^*(\scrR\scrR^*)^{-1}\scrA_G)^{-1}$ and used Lemma \ref{lem:bias_approx_error} to bound this ``bias operator" by the ``idealized bias operator" defined in Assumption \ref{assumps}: $\norm{\scrB\eta^* }_{L_2} \leq (1 + \frac{c}{1-c})\norm{\scrBbr\eta^* }_{L_2} \leq (1+\frac{c}{1-c})b$.
\end{proof}

Using this lemma and noting (as in the previous section) that $\norm*{\diag^{\perp}(\scrR \scrR^*)e_i}_2 \leq \alpha_U - \alpha_L$ and dividing by $b$, we conclude that term $B$ is bounded (within a constant) by $\parens*{\frac{\alpha_U}{\alpha_L} - 1}\sqrt{n} \asymp \frac{\alpha_U - \alpha_L}{\alpha_U + \alpha_L} \sqrt{n}$.

\subsection{Bias error in $G$}
For Term $C$, we first note that $|\scrP_{G}\etahat_{\setminus i, 0}(x_i) - \scrSbr\eta^*(x_i)|$ corresponds to the absolute error between the predictions of the $G$-projected estimator and its ``idealized" version, given by $\scrSbr\eta^*$. In Lemma \ref{lem:g_term}, we prove a high probability bound for $|\scrP_{G}\etahat_{0}(x) - \scrSbr\eta^*(x)|$, where $x$ is independent of all training points.

\begin{lemma}[Error of the $G$-projected estimator]\label{lem:g_term}
Let $x \in X$ be independent of the training data, and suppose Assumption \ref{assumps} holds. Then, with probability at least $1-\frac{1}{n^2}$,
\[
    \frac{1}{b}|\scrP_{G}\etahat_{0}(x) - \scrSbr\eta^*(x)| \lesssim \sqrt{\frac{p \log{n}}{n}} + \sqrt{p} \frac{\alpha_U - \alpha_L}{\alpha_U + \alpha_L} + c^2 \sqrt{p}.
\]
\end{lemma}

\begin{proof}[Proof of Lemma \ref{lem:g_term}]
First note that we can write
    $|\scrP_{G}\etahat_{0}(x) - \scrSbr\eta^*(x)| = |((\scrB - \scrBbr)\eta^*)(x)|,$
where $\scrB = (\scrI_G + \scrA_G^* (\scrR \scrR^*)^{-1} \scrA_G)^{-1} \eqqcolon (\scrI_G + \scrM)^{-1}
$
and
$
	\scrBbr = \parens*{ \scrI_G + \frac{n}{\alphabr} \scrT_G}^{-1} \eqqcolon (\scrI_G + \scrMbr)^{-1}.
$
Then note that
\begin{equation*}
	\begin{aligned}
	\scrB - \scrBbr
	&= \scrBbr(\scrMbr - \scrM) \scrB = \scrBbr(\scrMbr - \scrM) \scrBbr + \scrBbr(\scrMbr - \scrM) (\scrB - \scrBbr).
	\end{aligned}
\end{equation*}

So, we get
\begin{equation*}
    \begin{aligned}
        |\scrP_{G}\etahat_{0}(x) - \scrSbr\eta^*(x)| &\leq |(\scrBbr(\scrMbr - \scrM) \scrBbr \eta^*)(x)| + |(\scrBbr(\scrMbr - \scrM) (\scrB - \scrBbr) \eta^*)(x)|\\
    &\leq |(\scrBbr(\scrMbr - \scrM) \scrBbr \eta^*)(x)| + \norm{\scrBbr(\scrMbr - \scrM) (\scrB - \scrBbr) \eta^*}_{L_\infty}
    \end{aligned}
\end{equation*}

We then use the decomposition 
\begin{equation*}
    \begin{aligned}
    	\scrMbr - \scrM
    	&= \frac{n}{\alphabr} \scrT_G - \scrA_G^* (\scrR \scrR^*)^{-1} \scrA_G = \frac{n}{\alphabr} \parens*{ \scrT_G - \frac{1}{n}\scrA_G^* \scrA_G } + \scrA_G^* \parens*{ \frac{1}{\alphabr} I_n - (\scrR \scrR^*)^{-1} } \scrA_G,
     \end{aligned}
\end{equation*}
which, applied to the first term, yields
\begin{equation}
    \label{eq:bias_err_decomp}
    \begin{aligned}
        |\scrP_{G}\etahat_{0}(x) - \scrSbr\eta^*(x)| &\leq \abs*{\parens*{\scrBbr\frac{n}{\alphabr} \parens*{ \scrT_G - \frac{1}{n}\scrA_G^* \scrA_G } \scrBbr \eta^*}(x)}\\
        &\text{ }+ \abs*{\parens*{\scrBbr\scrA_G^* \parens*{ \frac{1}{\alphabr} I_n - (\scrR \scrR^*)^{-1} } \scrA_G \scrBbr \eta^*}(x)} + \norm*{\scrBbr(\scrMbr - \scrM) (\scrB - \scrBbr) \eta^*}_{L_\infty}\\
        &\leq \abs*{\parens*{\scrBbr\frac{n}{\alphabr} \parens*{ \scrT_G - \frac{1}{n}\scrA_G^* \scrA_G } \scrBbr \eta^*}(x)}\\
        &\quad + \norm*{\scrBbr\scrA_G^* \parens*{ \frac{1}{\alphabr} I_n - (\scrR \scrR^*)^{-1} } \scrA_G \scrBbr \eta^*}_{L_\infty} +  \norm*{\scrBbr(\scrMbr - \scrM) (\scrB - \scrBbr) \eta^*}_{L_\infty}
    \end{aligned}
\end{equation}

We obtain the desired bound by separately bounding each of the three terms in (\ref{eq:bias_err_decomp}).

For the first term, note that
\begin{equation*}
    \begin{aligned}
	\scrBbr \cdot \frac{n}{\alphabr} \parens*{ \scrT_G - \frac{1}{n}\scrA_G^* \scrA_G } \scrBbr
	&= \frac{n}{\alphabr} \parens*{ \scrI_G + \frac{n}{\alphabr} \scrT_G}^{-1} \parens*{ \scrT_G - \frac{1}{n}\scrA_G^* \scrA_G } \scrBbr \\
	&= \parens*{ \frac{\alphabr}{n} \scrT_G^{-1} + \scrI_G }^{-1} \parens*{\scrI_G - \frac{1}{n} \scrC^* \scrC} \scrBbr.
    \end{aligned}
\end{equation*}
Rather than bounding an $L_\infty$ norm,
we directly analyze the application of this operator to $\eta^*$.
Let $\scrD \coloneqq \parens*{ \frac{\alphabr}{n} \scrT_G^{-1} + \scrI_G }^{-1}$; note that $\scrD$ is a contraction on $G$ (i.e., $\norm{Df}_{L_2} \leq \norm{f}_{L_2}$ for $f \in G$). Furthermore, let $\delta^G_{x_i} = \scrT_G^{-1} k_{x_i}^G$. Then, for fixed $f \in G$ and $x \in X$, we have
\begin{equation*}
    \begin{aligned}
    	\abs*{\parens*{\scrD \parens*{\scrI_G - \frac{1}{n} \scrC^* \scrC} f}(x)}
    	&= \abs*{\frac{1}{n} \sum_{i=1}^n ( f(x_i) \scrD \delta_{x_i}^G(x)  - (\scrD f)(x))},
    \end{aligned}
\end{equation*}
which is a sum of independent zero-mean random variables.
Specifically, each of the variables $f(x_i) (\scrD \delta_{x_i}^G)(x) - (\scrD f)(x)$ is zero-mean and has variance
\[
	\E_{x_i} \abs{f(x_i) (\scrD \delta_{x_i}^G)(x) - (\scrD f)(x)}^2 \leq \E_{x_i} \abs{f(x_i) (\scrD \delta_{x_i}^G)(x)}^2 \leq \norm{f}_\infty^2 \norm{\scrD \delta^G_x}_{L_2}^2 \lesssim p \norm{f}_\infty^2,
\]

where we use the fact $D$ is a contraction and the BOS assumption in the last inequality. Moreover, each term is bounded as
$
\abs{f(x_i) (\scrD \delta_{x_i}^G)(x) - (\scrD f)(x)} \lesssim \norm{f}_\infty p.
$
So, Bernstein's inequality (e.g., \cite[Section 2.8]{Vershynin2018}) gives, with probability at least $1-\frac{1}{n^2}$,
\[
	\abs*{\parens*{\scrD \parens*{\scrI_G - \frac{1}{n} \scrC^* \scrC} f}(x)} \lesssim \sqrt{\frac{p \log{n}}{n}} \norm{f}_\infty + \frac{p\log{n}}{n} \norm{f}_\infty \lesssim \sqrt{\frac{p \log{n}}{n}} \norm{f}_\infty.
\]
We bound the first term by applying this with $f = \scrBbr \eta^*$ and recalling that $\norm{\scrBbr \eta^*}_\infty = b$.

For the second term of (\ref{eq:bias_err_decomp}), it is proved in Lemma 7 of \cite{McRae2022} (which we state as Lemma \ref{lem:bias_approx_error}) that
\begin{equation*}
	\norm*{\scrBbr \scrA_G^* \parens*{ \frac{1}{\alphabr} I_n - (\scrR \scrR^*)^{-1}} \scrA_G}_{L_2}
	\leq \frac{\alpha_U - \alpha_L}{\alpha_U + \alpha_L}\parens*{ 1 + \frac{1}{n} \norm{\scrC^* \scrC - n \scrI_G}_{L_2} } \asymp \frac{\alpha_U - \alpha_L}{\alpha_U + \alpha_L},
\end{equation*}
where the last expression follows by the second part of Assumption \ref{assumps}. Recall that, by the BOS assumption, for all $f \in G$, $\norm{f}_{L_\infty} \lesssim \sqrt{p} \norm{f}_{L_2}$. So, applying this to $f = \scrBbr \eta^*$ and noting that $\norm{\scrBbr \eta^*}_{L_2} \leq b$,
we get

\[
	\norm*{\scrBbr \scrA_G^* \parens*{ \frac{1}{\alphabr} I_n - (\scrR \scrR^*)^{-1}} \scrA_G \scrBbr \eta^*}_{L_\infty}
	\lesssim \sqrt{p} \frac{\alpha_U - \alpha_L}{\alpha_U + \alpha_L} b .
\]

To bound the third term, note that $\scrBbr (\scrMbr - \scrM)$ is precisely the operator $\scrQ$ from Lemma \ref{lem:bias_approx_error}. Hence,$
\norm{\scrBbr (\scrMbr - \scrM)}_{L_2}  \leq c.
$
Furthermore, we have (again using the result of Lemma \ref{lem:bias_approx_error})
\begin{equation}
    \label{eq:bias-norm}
    \begin{aligned}
    	\norm{(\scrBbr - \scrB)\eta^*}_{L_2} \leq \frac{c}{1-c} \norm{\scrBbr \eta^*}_{L_2} \leq\frac{c}{1-c} \norm{\scrBbr \eta^*}_{L_\infty} \leq \frac{bc}{1-c}
    \end{aligned}
\end{equation}
Now, using the BOS assumption, the third term in (\ref{eq:bias_err_decomp}) can be bounded as:
\begin{equation*}
	\norm{(\scrBbr(\scrMbr - \scrM) (\scrB - \scrBbr) \eta^*}_{L_\infty}
	\leq \sqrt{p} \cdot \norm{\scrBbr (\scrMbr - \scrM)}_{L_2} \norm{(\scrBbr - \scrB)\eta^*}_{L_2} \leq \frac{bc^2 \sqrt{p}}{1-c} \lesssim bc^2 \sqrt{p}
\end{equation*}


Combining all of the above bounds, we finally arrive at
\begin{equation*}
    |\scrP_{G}\etahat_{0}(x) - \scrSbr\eta^*(x)| \lesssim \sqrt{\frac{p \log{n}}{n}}\cdot b +  \frac{\alpha_U - \alpha_L}{\alpha_U + \alpha_L}\cdot b \sqrt{p} + bc^2 \sqrt{p}
\end{equation*}
with probability at least $1-\frac{1}{n^2}$.

\end{proof}

To bound term $C$ in (\ref{eq:svm-cond}), we can apply this lemma to the $n$ leave-one-out training sets and take a union bound to get the final bound, with probability at least $1-\frac{1}{n}$. Combining the bounds on terms $A$, $B$, and $C$ in the previous three subsections yields Theorem $\ref{thm:bos}$.

\subsection{Proofs for sub-Gaussian features}\label{sec:subg_proofs}
Suppose that instead of satisfying the BOS condition, the features $v_{\ell}(x)$ are sub-Gaussian and independent. In this case, we first note that for all $f \in \scrH$, $f(x) = \langle k_{x}, f \rangle_{\scrH}$ is a sub-Gaussian random variable with norm at most $\norm{f}_{L_2}$, up to constants. Using this observation, we can obtain sharp bounds for each term in (\ref{eq:svm-cond}).

For term $A$, we note that the statement of Lemma \ref{lem:noise_term} also holds in the independent, sub-Gaussian case, with probability at least $1-\frac{1}{n^2} - e^{-p}$ (this fact is verified in the proof of the lemma). So, after applying Lemma \ref{lem:noise_term} to each of the $n$ leave-one-out sets, we obtain the same upper bound for term $A$ as in the BOS case, with probability $1-ne^{-p}-\frac{1}{n}$.

For term $B$ we note that the term $\abs{\scrP_{G^\perp} \etahat_{\setminus i, 0} (x)}$ is sub-Gaussian with norm at most $\norm*{\scrP_{G^\perp} \scrA^* (\scrA \scrA^*)^{-1} \scrA \eta^*}_{L_2}$. We bound this in the following lemma.

\begin{lemma}[$G^\perp$ error term for sub-Gaussian features]\label{lem:subg_g_perp_term}
Suppose Assumption $\ref{assumps}$ holds for independent, sub-Gaussian features. Then, 
\[
\norm*{\scrP_{G^\perp} \scrA^* (\scrA \scrA^*)^{-1} \scrA \eta^*}_{L_2} \lesssim \sqrt{\frac{n \lambda_{p+1}}{\alpha_L}}b.
\]
\end{lemma}
\begin{proof}[Proof of Lemma \ref{lem:subg_g_perp_term}]
We proceed in a similar fashion to in the proof of Lemma \ref{lem:g_perp_term}.
\begin{equation*}
    \begin{aligned}
        \norm*{\scrP_{G^\perp} \scrA^* (\scrA \scrA^*)^{-1} \scrA \eta^*}_{L_2} &\leq \norm{\scrI_{G^\perp}}_{\scrH \to L_2} \norm{\scrR^*(\scrA \scrA^*)^{-1}\scrA \eta^*}_{\scrH}\\
        &= \norm{\scrI_{G^\perp}}_{\scrH \to L_2} \norm{\scrR^*(\scrR \scrR^* + \scrA_G \scrA_G^*)^{-1}\scrA \eta^*}_{\scrH}\\
        &= \norm{\scrI_{G^\perp}}_{\scrH \to L_2} \norm{\scrR^*(\scrR \scrR^*)^{-1}\scrA_G \scrB \eta^*}_{\scrH}\\
        &\leq \norm{\scrI_{G^\perp}}_{\scrH \to L_2} \norm{\scrR^*(\scrR \scrR^*)^{-1}}_{\ell_2 \to \scrH} \norm{\scrA_G}_{L_2 \to \ell_2} \norm{\scrB \eta^*}_{L_2}\\
        &\lesssim \sqrt{\lambda_{p+1}}\frac{1}{\sqrt{\alpha_L}} \sqrt{n}  b = \sqrt{\frac{n \lambda_{p+1}}{\alpha_L}}b.
    \end{aligned}
\end{equation*}
In the second to last line, we used the result of Lemma \ref{lem:bias_approx_error} to bound $\norm{\scrB \eta^*}_{L_2}$.
\end{proof}

Using a sub-Gaussian tail bound, term $B$ is thus bounded by 
$\sqrt{\frac{\lambda_{p+1}n \log{n} }{\alpha_L}}$
with probability at least $1-\frac{1}{n^2}$. We then apply this to each of the $n$ leave-one-out sets to obtain the bound on term $B$ with probability $1-\frac{1}{n}$.

For term $C$, we note, as in the proof of Lemma \ref{lem:g_term}, that 
\begin{equation*}
    |\scrP_{G}\etahat_{0}(x) - \scrSbr\eta^*(x)| = |((\scrB - \scrBbr)\eta^*)(x)|,
\end{equation*}
where $\scrB := (\scrI_{G} + \scrA_G^*(\scrR\scrR^*)^{-1}\scrA_G)^{-1}$ (this follows from the definition of $\etahat$ and the pushthrough identity). This term hence has sub-Gaussian norm at most
\begin{equation*}
    \norm{(\scrB - \scrBbr)\eta^*}_{L_2} \leq \frac{bc}{1-c} \lesssim bc,
\end{equation*}
where we have used the result of Lemma \ref{lem:bias_approx_error} in the Appendix.
So, using a tail-bound, we have that 
\[
|\scrP_{G}\etahat_{0}(x) - \scrSbr\eta^*(x)| \lesssim bc \sqrt{\log{n}}
\]
with probability at least $1-\frac{1}{n^2}$ (so this holds for all the leave-one-out sets with probability $1-\frac{1}{n}$). Combining the above bounds yields Theorem \ref{thm:subg}.

\section{Simulations and discussion}

\begin{figure}[tbhp]
    \centering
    \subfloat[$q = -0.4$]{\label{fig:a}\includegraphics{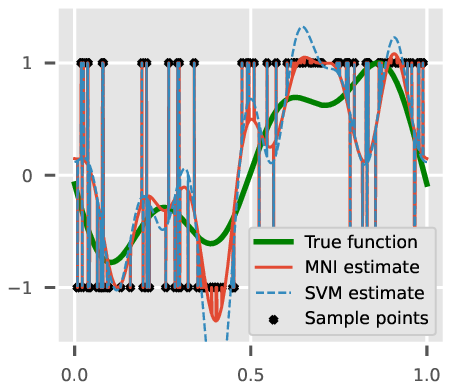}}
    \subfloat[$q = 0.4$]{\label{fig:b}\includegraphics{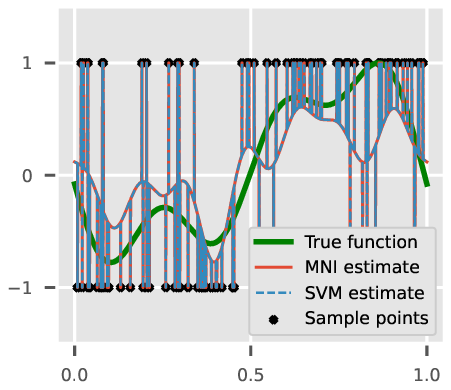}}
    \subfloat[$q = 0.8$]{\label{fig:c}\includegraphics{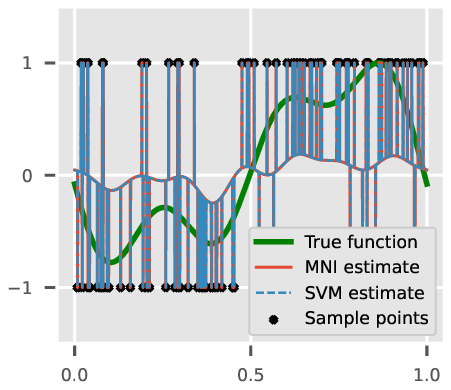}}
    \caption{SVM and MNI solutions with Fourier features in the bi-level ensemble, for various levels of effective overparameterization.}
    \label{fig:fig1}
\end{figure}

To verify our theoretical findings, we now perform simulations of the MNI and SVM procedures under the bi-level model. We consider the case where the input space $X$ is the interval $[0,1]$ with the uniform measure. Furthermore, we consider the Fourier featurization, specified by the eigenfunctions%
\footnote{For simplicity, we use complex exponential eigenfunctions, but we will always use complex-conjugate--symmetric coefficients, resulting in real functions. We could use purely real sines and cosines as our eigenfunctions at the expense of somewhat more complicated notation.}
$v_\ell(x) = e^{j2\pi \ell x}$ for $\ell = -d, \dots, d$. Performing MNI or SVM with these features and the bi-level weighting on features corresponds to the kernel function
\[
k(x, y) = \sum_{\ell = -d}^d \lambda_{\ell} v_\ell(x) \overline{v_{\ell}(y)} = \sum_{\ell = -d}^d \lambda_{\ell} e^{j2\pi \ell (x - y)}
\]
where, the $\lambda_{\ell}$ are given in Equation \ref{eq:bilevel} for $\ell > 0$, and $\lambda_{-\ell} = \lambda_\ell$. For all experiments, we use a training size of $n=100$ (note that the choice of $n$ affects the kernel eigenvalues). 

We first generate a random ground-truth function $\eta^*$ as a linear combination of $\{v_{\ell}\}_{\ell = -3}^3$, with standard normal coefficients, scaled so that $\eta^*(x) \in [-1,1]$. For each $x_i$ in the training set, we generate $y_i$ according to the rule $\P(y_i = 1) = \frac{1+\eta^*(x_i)}{2}$. We then compute and plot the MNI and SVM estimates for three choices of bi-level parameters: $(\beta, r, q) = (3.2, 0.4, -0.4)$,  $(3.2, 0.4, 0.4)$, and $(3.2, 0.4, 0.8),$ corresponding to different levels of effective overparameterization. The results are shown in Figure \ref{fig:fig1}. Based on Corollary \ref{cor:bos}, the two estimates should coincide with high probability when $q > \frac{1-0.4}{2} = 0.3$. Indeed, from Figures \ref{fig:b} and \ref{fig:c}, we see that this is the case for the second and third choice of parameters. Our results do not draw any conclusions about the occurrence of SVP in the range $q < 0.3$, but we see that in this example, SVP does not occur for a very low degree of effective overparameterization, as shown in Figure \ref{fig:a}. 

To measure the effect of the bi-level ensemble parameters, we also plot a heatmap of SVP occurrence for various values of $r$ (the number of favored features) and $q$ (the effective overparameterization). Here, we fix $n=100$, $\beta = 3.2$, and the ground-truth function $\eta^*$ (generated as in Fig. \ref{fig:fig1}), and we plot the proportion of trials (out of 25) that SVP occurs for each $(r,q)$. A heatmap of these results is provided in Figure \ref{fig:fig2}, along with an overlay of the regime for which Corollary \ref{cor:bos} predicts SVP. We note that, at least for the case of Fourier features, there exist values of $(r,q)$ for which SVP occurs but which are not predicted by our theoretical results; understanding SVP in these regimes (e.g., $\frac{2}{3} < r < 1$) (and more generally, providing sharp lower bounds) is an interesting and important direction for future work. 

\begin{figure}[tbhp]
    \centering
    \includegraphics{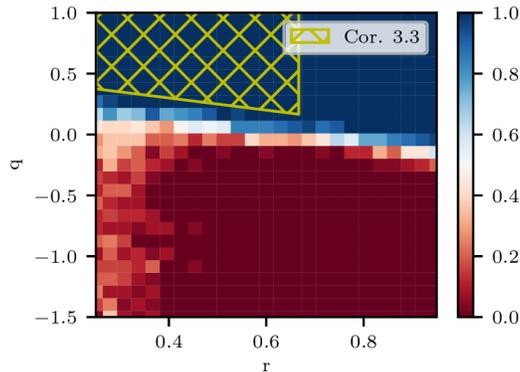}
    \caption{Heatmap showing the proportion of trials (out of 25) that SVP occurs for varying parameter choices in the Fourier bi-level ensemble, and overlay of the SVP regime predicted by our results.}
    \label{fig:fig2}
\end{figure}


\section*{Acknowledgments} 
We gratefully acknowledge the support of the NSF, Adobe Research, Amazon Research and Google Research in carrying out this work.
CK was supported by NSF Graduate Research Fellowship DGE-2039655. AM and MD were supported by NSF grants DMS-2134037 and CCF-2107455. VM was supported by NSF CAREER award CCF-239151, NSF award IIS-2212182, an Adobe Data Science Research Award, an Amazon Research Award, and a Google Research Colabs award.

\appendix

\section{Notation}
In Table \ref{tab:notation}, we review the main symbols which are used in our results and proofs.

\label{app:notation} 
\begin{table}[tbhp]    
	\centering
	\caption{Notation}
	\label{tab:notation}
	\begin{tabular}{|c|l|p{13em}|}
		\hline
		\bf Symbol(s) & \bf Definition(s) & \bf Description \\ \hline
		$k_x$ & $k_x = k(\cdot, x)$ & Kernel function centered at $x$ \\
		$\scrT$ & $\scrT(f) = \int f(x) k_x \ d\mu(x)$ & Integral operator of kernel $k$ \\
		$\{(\lambda_\ell, v_\ell)\}_{\ell=1}^\infty$ & $\scrT(f) = \sum_{\ell=1}^\infty \lambda_\ell \ipsub{f}{v_\ell}{L_2} v_\ell,~\lambda_1 \geq \lambda_2 \geq \cdots$ & Eigenvalue decomposition of $\scrT$ \\
		$\scrA$ & $\scrA(f) = \begin{bmatrix*} f(x_1) \\ \vdots \\ f(x_n) \end{bmatrix*}$ & Sampling operator from $\scrH$ to $\R^n$ \\
		$\scrA^*$ & $\scrA^*(z) = \sum_{i=1}^n z_i k_{x_i}$ & Adjoint of $\scrA$ w.r.t\ $\scrH$ and $\ell_2$ inner products \\
		$G$, $G^\perp$ & $G = \spn\{v_1, \dots, v_p\}$ & Span of first $p$ eigenfunctions of $\scrT$ (and its complement) \\
		$\scrI$ ($\scrI_G$) & & Identity operator (restricted to $G$) \\
		$\scrT_G, \scrT_{G^\perp}$ & $\scrT_G = \scrT \scrP_G$, $\scrT_{G^\perp} = \scrT \scrP_{G^\perp}$ & $\scrT$ restricted to $G$ and $G^\perp$ \\
		$\scrA_G$, $\scrR$ & $\scrA_G = \scrA \scrP_G, \scrR = \scrA \scrP_{G^\perp}$ & Restrictions of sampling operator to $G$, $G^\perp$ \\
		$\scrC, \scrC^*$ & $\scrC = \scrA_G$, $\scrC^* = \scrT_G^{-1} \scrA_G^*$ & Sampling operator and its adjoint on $G$ w.r.t.\ $L_2$ inner product on $G$ \\
		$\alpha_L, \alpha_U$ & $\alpha_L I_n \preceq \alpha I_n + \scrR \scrR^* \preceq \alpha_U I_n$ & Lower and upper bounds for the residual Gram matrix \\
		$c$ & $\frac{\alpha_U - \alpha_L}{\alpha_U + \alpha_L} + \frac{2}{n}\norm{\scrC_{\setminus i}^* \scrC_{\setminus i} - n \scrI_G}_{L_2} \leq c, \forall i \in [n]$ & Upper bound on feature concentration in $G$\\
		$b$ & $b = \norm{\scrBbr \eta^*}_{\infty}$ & Infinity norm of idealized bias applied to $\eta^*$\\
		$\alphabr,\alphatl$ & $\alphabr = \frac{2 \alpha_U \alpha_L}{\alpha_U + \alpha_L}, \alphatl = \frac{\alpha_U + \alpha_L}{2}$ & Harmonic and arithmetic means of $\alpha_U, \alpha_L$ \\
		$\scrB$ & $\scrB = (\scrI_G + \scrA_G^* (\alpha + \scrR \scrR^*)^{-1} \scrA_G)^{-1}$ & Bias operator on $G$ \\
		$\scrS$ & $\scrS = \scrI_G - \scrB$ & Kernel regression operator (``survival'') on $G$ \\
		$\scrBbr$ & $\scrBbr = \parens*{\scrI_G + \frac{n}{\alphabr} \scrT_G}^{-1}$ & Idealized approximation to bias $\scrB$ \\
		$\scrSbr$ & $\scrSbr = \scrI_G - \scrBbr = \frac{n}{\alphabr} \scrT_G \parens*{\scrI_G + \frac{n}{\alphabr} \scrT_G}^{-1}$ & Idealized approximation to survival $\scrS$ \\
		\hline
	\end{tabular}
\end{table}

\section{Concentration bounds and technical lemmas}\label{app:bounds}
In this section we state some useful concentration bounds which can be used to verify the operator concentration conditions in Assumption \ref{assumps}. These inequalities and their proofs appear in \cite{McRae2022}. Similar concentration bounds for sub-Gaussian/independent features can also be found in \cite{Bartlett2020, Tsigler2020}.

\begin{lemma}[General residual Gram matrix concentration] Let $k^R(x, y)$ be the restriction of the kernel $k$ to $G^\perp$. Then, 
    \label{lem:R_conc_generic}
	\begin{equation*}
		\E \norm{\scrR \scrR^* - (\tr \scrT_{G^\perp}) I_n}^2
		\lesssim n^2 \tr(\scrT_{G^\perp}^2)
		+ \norm{k^R(\cdot, \cdot) - \tr \scrT_{G^\perp} }_\infty^2,
	\end{equation*}
	where $\norm{k^R(\cdot, \cdot) - \tr \scrT_{G^\perp} }_\infty = \sup_x \{\abs{k^R(x, x) - \tr \scrT_{G^\perp} } \}$.
\end{lemma}

\begin{lemma}[Residual Gram matrix concentration for independent, sub-Gaussian features]
	\label{lem:R_conc_ind}
	Suppose the features $\{v_\ell(x)\}_{\ell \geq 1}$ are zero-mean, independent, and sub-Gaussian.
	Then, for $t > 0$, with probability at least $1 - e^{-t}$,
	\[
		\norm{\scrR \scrR^* - (\tr \scrT_{G^\perp}) I_n } \lesssim \sqrt{(n + t) \tr(\scrT_{G^\perp}^2)} + (n + t) \lambda_{p+1}.
	\]
\end{lemma}

\begin{lemma}[Sampling operator concentration on $G$ in a BOS]
	\label{lem:C_conc_BOS}
	Suppose the kernel features $\{v_\ell(x)\}_{\ell \geq 1}$ satisfy the BOS condition with constant $C \geq 1$.
	Then, for $t > 0$, with probability at least $1 - e^{-t}$,
	\[
	\frac{1}{n} \norm{ \scrC^* \scrC - n \scrI_G}_{L_2} \lesssim \sqrt{\frac{C p (t + \log p)}{n}} + \frac{Cp (t + \log p)}{n}.
	\]
\end{lemma}

\begin{lemma}[Sampling operator concentration on $G$ for independent, sub-Gaussian features]\label{lem:C_conc_ind}
    Suppose the features $\{v_\ell(x)\}_{\ell \geq 1}$ are zero-mean, independent, and sub-Gaussian. Then, with probability at least $1 - e^{-t}$,
    \[\frac{1}{n} \norm*{\scrC^* \scrC - n\scrI_G}_{L_2}
    \lesssim \sqrt{\frac{p+t}{n}} + \frac{p+t}{n}.\]
\end{lemma}

We next state the following approximation result, proven in \cite{McRae2022}, which we use frequently to prove various technical lemmas.
\begin{lemma}[Lemma 7 in \cite{McRae2022}]
	\label{lem:bias_approx_error}
	Let $\scrB := (\scrI_{G} + \scrA_G^*(\scrR\scrR^*)^{-1}\scrA_G)^{-1}$ and $\scrBbr = \parens*{\scrI_G + \frac{n}{\alphabr} \scrT_G}^{-1} $. Then, under the conditions of Assumption \ref{assumps},
	\[
		\scrB - \scrBbr
		= \parens*{\sum_{k=1}^\infty (-\scrQ)^k } \scrBbr,
	\]
	for $\scrQ \coloneqq \parens*{ \scrT_G^{-1} + \frac{n}{\alphabr} \scrI_G }^{-1} \parens*{ \scrC^* (\scrR \scrR^*)^{-1} \scrC - \frac{n}{\alphabr} \scrI_G }$ satisfying $\norm{\scrQ}_{L_2} \leq c$. Here, $c < 1$ is an upper bound on the quantity $\frac{\alpha_U - \alpha_L}{\alpha_U + \alpha_L} + \frac{2}{n} \norm*{ \scrC^*\scrC - n \scrI_G}_{L_2}$.
\end{lemma}

Note that by applying the formula for the sum of an infinite geometric series, the above result implies that $\norm{(\scrB-\scrBbr) \eta^*}_{L_2} \leq \frac{c}{1-c} \norm{\scrBbr \eta^*}_{L_2}$ and $\norm{\scrB \eta^*}_{L_2} \leq \parens*{1 + \frac{c}{1-c}} \norm{\scrBbr \eta^*}_{L_2}$. 

\section{Proofs of Corollaries \ref{cor:bos} and \ref{cor:subg}}\label{app:bilevel_proof}
\begin{proof}[Proof of Corollary \ref{cor:bos}]
We bound each term for the bi-level ensemble as in \cite[Corollary 1]{McRae2022}. First, note that

 
\[
    n\sqrt{\tr(\scrT_{G^\perp}^2)} \asymp n\cdot \parens{\parens{n^\beta - n^r} \cdot n^{-2\beta +2r +2q}}^{1/2} \asymp n^{1-\frac{\beta}{2} +r+q}.
\]
By the thin-shell assumption that $\sup_{x}\abs*{\sum_{\ell > p}(v_{\ell}(x)^2 - 1)} \lesssim n^{\beta/2 + 1}$, 

\begin{equation*}
    \begin{aligned}
        \sup_{x}\abs{k^R(x, x) - \tr \scrT_{G^\perp} } &= \sup_{x}\abs*{\sum_{\ell > p}\lambda_\ell (v_{\ell}(x)^2 - 1)} = n^{-\beta + r + q}\sup_{x}\abs*{\sum_{\ell > p}(v_{\ell}(x)^2 - 1)}\\
        &\lesssim n^{-\beta + r + q}n^{\beta/2 + 1} = n^{1 - \beta/2 + r + q},
        \end{aligned}
\end{equation*}
so Lemma \ref{lem:R_conc_generic} implies that in the limit as $n \to \infty$, we can choose $\alpha_L = \lambda_{min}(\scrR \scrR^*)$ and $\alpha_U = \lambda_{max}(\scrR \scrR^*)$ with $\alpha_U, \alpha_L, \alphabr \asymp n^{r+q}$ as long as $1-\frac{\beta}{2} +r+q < r+q$, which is true by our assumption that $\beta > 2$. 
To be more precise, \Cref{lem:R_conc_generic} gives us an expectation bound, which we can then use with Markov's inequality to get a high-probability bound. We can obtain the stated bounds within factors of $n^\epsilon$ for arbitrarily small $\epsilon > 0$, and the probability of these bounds holding goes to $1$ as $n \to \infty$. We can now bound the term $c$ from Assumption \ref{assumps}. Lemma \ref{lem:R_conc_generic} implies that
\begin{equation*}
    \frac{\alpha_U - \alpha_L}{\alpha_U + \alpha_L} \lesssim \frac{1}{\alphabr} n\sqrt{\tr(\scrT_{G^\perp}^2)} \asymp n^{1-\frac{\beta}{2}}.
\end{equation*}

Then, applying Lemma \ref{lem:C_conc_BOS} for each of the $n$ leave-one-out sets allows us to bound $c$ as
\[
c \lesssim \frac{\alpha_U - \alpha_L}{\alpha_U + \alpha_L} + \frac{2}{n}\norm{\scrC_{\setminus i}^* \scrC_{\setminus i} - n \scrI_G}_{L_2} \lesssim n^{1-\frac{\beta}{2}} + n^{\frac{r-1}{2}}\sqrt{\log{n}}. 
\]

Now, note that for the bi-level ensemble, we have $b = \norm*{\frac{\alphabr}{\alphabr + n} \scrI_G \eta^*}_{L_\infty} = \frac{\alphabr}{\alphabr + n} \norm{\eta^*}_{L_\infty} \asymp \frac{\alphabr}{\alphabr + n} \asymp \min\{1, n^{r+q-1}\}$. We consider two cases (ignoring log factors, which are negligible compared to powers of $n$ as $n \to \infty$):
\begin{enumerate}
    \item If $q > 1-r$, $b \asymp 1$, so the conditions for SVP from Theorem \ref{thm:bos} can be written as
    \[
    n^{(r-1)/2} + n^{1-\beta/2}n^{1/2} + \parens{n^{2-\beta} + n^{r-1}}n^{r/2} \ll 1
    \]
    One can easily check that this holds as $n \to \infty$ when $r<2/3$ and $\beta > 3$.
    \item If $q \leq 1-r$, then $b \asymp n^{r+q-1}$, so the condition is
    \[
    n^{1-r-q}n^{(r-1)/2} + n^{1-\beta/2}\parens{n^{1-r-q} +  n^{1/2}} + \parens{n^{2-\beta} + n^{r-1}}n^{r/2} \ll 1
    \]
    This holds as $n \to \infty$ for $r<2/3$, $q>\frac{1-r}{2}$, and $\beta > 3$.
\end{enumerate}
Combining the above two cases yields the result.

\end{proof}

\begin{proof}[Proof of Corollary \ref{cor:subg}]
The proof of this statement is similar to that of Corollary \ref{cor:bos}, but it uses sharper concentration bounds from Lemmas \ref{lem:R_conc_ind} and \ref{lem:C_conc_ind} for independent, sub-Gaussian features. Again, we first bound the quantities $\alpha_U - \alpha_L$ and $c$. Lemma \ref{lem:R_conc_ind} implies that with probability at least $1-n^{-1}$, 

\[
	\norm{\scrR \scrR^* - (\tr \scrT_{G^\perp}) I_n } \lesssim n^{1/2 - \beta/2 + r + q} + n^{1-\beta+r+q}.
\]
Since $\beta > 1$, we can again choose $\alpha_L = \lambda_{min}(\scrR \scrR^*)$ and $\alpha_U = \lambda_{max}(\scrR \scrR^*)$ with $\alpha_U, \alpha_L, \alphabr \asymp n^{r+q}$. Furthermore,
\begin{equation*}
    \frac{\alpha_U - \alpha_L}{\alpha_U + \alpha_L} \lesssim \frac{1}{n^{r+q}} n^{(1-\beta)/2 + r + q} \asymp n^{(1-\beta)/2},
\end{equation*}
and Lemma \ref{lem:C_conc_ind} (again applied as a union bound over the $n$ leave-one-out sets) implies that we can take
$
    c \asymp n^{(1-\beta)/2} + n^{\frac{r-1}{2}} \gtrsim \frac{\alpha_U - \alpha_L}{\alpha_U + \alpha_L} + \frac{2}{n}\norm{\scrC_{\setminus i}^* \scrC_{\setminus i} - n \scrI_G}_{L_2}.
$
We again consider the two cases in the previous corollary (again omitting log factors):

\begin{enumerate}
    \item If $q > 1-r$, $b \asymp 1$, the conditions for SVP from Theorem \ref{thm:subg} can be written as
    \[
    n^{(r-1)/2} + n^{(1-\beta)/2} + \sqrt{\frac{n^{-\beta+r+q+1}}{n^{r+q}}} +  \parens{n^{(1-\beta)/2} + n^{\frac{r-1}{2}}} \ll 1
    \]
    Here, the left hand side is always decreasing in $n$ for $\beta > 1$ and $r<1$.
    
    \item If $q \leq 1-r$, then $b \asymp n^{r+q-1}$, so the condition is
    \[
    n^{\frac{r-1}{2}+ 1-r-q} + n^{1-r-q} \cdot n^{(1-\beta)/2} + \sqrt{\frac{n^{-\beta+r+q+1}}{n^{r+q}}} +  \parens{n^{(1-\beta)/2} + n^{\frac{r-1}{2}}} \ll 1
    \]
    In this case, the left hand side is decreasing for $q>\frac{1-r}{2}$, and $\beta > 3-2r-2q$.
\end{enumerate}
Combining the above two cases, we arrive at the desired result.
\end{proof}

\section{Proof of Corollary \ref{cor:gen}}\label{app:gen_proof}
\
To prove this corollary, we first characterize the range of parameters $(n, \beta, q, r)$ for which the excess classification risk of the MNI solution goes to $0$ in the limit as $n \to \infty$. In the BOS case, it was proven in Corollary 1 of \cite{McRae2022} that this occurs in two settings:

\begin{enumerate}
    \item $q<1-r$, $\beta > 2$ (for which regression is also consistent)
    \item $1-r < q< \frac{3}{2}(1-r)$ and $\beta > 2(r+q)$ (for which classification is consistent but regression is not). 
\end{enumerate}

Combining this with the conditions for SVP in bounded orthonormal systems from \Cref{cor:bos}, we can immediately conclude that the SVM solution is classification-consistent in the regime where both conditions hold (i.e., when SVP occurs and when the MNI is consistent). The first part of Corollary \ref{cor:gen} follows.

Now, to prove the second part of the corollary, we first modify the proof of Corollary 1 of \cite{McRae2022} using sharper concentration bounds for independent, sub-Gaussian features.
We first recall the following expression of the excess risk from \cite{McRae2022}: $\scrE \leq \frac{\norm{\etahat_r}_{L_2}}{s},$
where $s>0$ is a parameter we can choose and $\etahat = s \eta^* + \etahat_r$. 

Note that the idealized survival operator in the bi-level ensemble is given by $\scrSbr = \frac{n}{n + \alphabr} \scrI_G \asymp \frac{n}{n + n^{r+q}} \scrI_g \asymp \min(1, n^{1-r-q}) \scrI_G$. So, we will apply the above bound on the excess classification risk with the choice $s = \min(1, n^{1-r-q})$. Then, by definition, $\etahat_r = \epsilon + \etahat_0 - \scrSbr\eta^*$,
so
\[
\norm{\etahat_r}_{L_2} \leq \norm{\epsilon}_{L_2} + \norm{\etahat_0 - \scrSbr\eta^*}_{L_2},
\]
The first term can be bounded by Theorem 2 of \cite{McRae2022}, which yields
\[
\norm{\epsilon}_{L_2} \lesssim \sqrt{\frac{p}{n}} + \parens*{\frac{n \sum_{\ell>p} \lambda_\ell^2}{\parens*{\sum_{\ell >p} \lambda_\ell}^2}}^{1/2} \asymp n^{(r-1)/2} + \parens*{\frac{nn^{\beta}n^{-2\beta +2r +2q}}{n^{2\beta}n^{-2\beta+2r+2q}}}^{1/2} = n^{(r-1)/2} + n^{(1-\beta)/2},
\]
where we have used the asymptotic scaling derived in the proof of Corollary \ref{cor:subg}.

Furthermore, we can bound $\norm{\etahat_0 - \scrSbr\eta^*}_{L_2}$ with Lemma 6 of \cite{McRae2022}:
\begin{equation*}
\begin{aligned}
\norm{\etahat_0 - \scrSbr\eta^*}_{L_2} &\lesssim \parens*{c + \sqrt{\frac{n\lambda_{p+1}}{\alphabr}}} \min\parens*{\lambda_1, \frac{\alphabr}{n \sqrt{\lambda_p}}, \sqrt{\frac{\alphabr}{n}}}\norm{\eta^*}_\scrH\\
&\lesssim \parens*{n^{(1-\beta)/2} + n^{\frac{r-1}{2}} + n^{(1-\beta)/2}} \min\parens*{1, n^{r+q-1}, n^{(r+q-1)/2}}\\
&\lesssim \parens*{n^{(1-\beta)/2} + n^{\frac{r-1}{2}}} \min\parens*{1, n^{r+q-1}}
\end{aligned}
\end{equation*}
So, we can obtain the risk bound
\[
\scrE \lesssim \max(1, n^{r+q-1})\parens*{n^{(r-1)/2} + n^{(1-\beta)/2} + \parens*{n^{(1-\beta)/2} + n^{\frac{r-1}{2}}} \min\parens*{1, n^{r+q-1}}}
\]
From this we obtain the following:
\begin{enumerate}
    \item If $q<1-r$, then $\scrE \to 0$ (and, in fact $s \to 1$).
    \item if $q>1-r$, then $\scrE \to 0$ if $q<\frac{3}{2}(1-r)$ and $q<1 + \frac{\beta-1}{2} -r$.
\end{enumerate}
Combining this range of allowable parameters with the conditions for SVP in \Cref{cor:subg} yields the second part of the corollary. Again, the case of $q \leq \frac{1 -r}{2} < 1 - r$ is covered by previous results as discussed in \cite{Hsu2021}.


\printbibliography
\end{document}